\documentclass[twoside,11pt]{article}
\usepackage{jair, theapa, rawfonts}

\usepackage{amsmath}
\usepackage{amssymb}
\usepackage{mathtools}
\usepackage{amsthm}
\usepackage{bbm}
\usepackage[ruled,linesnumbered,boxed]{algorithm2e}
\usepackage{multirow}
\usepackage{subfigure}
\usepackage{graphicx}
\usepackage{comment}

\newtheorem{theorem}{Theorem}
\newtheorem{lemma}{Lemma}
\newtheorem{corollary}{Corollary}
\theoremstyle{definition}
\newtheorem{definition}{Definition}
\newtheorem{assumption}{Assumption}
\newtheorem{remark}{Remark}
\newtheorem{fact}{Fact}

\begin{document}
	
	\title{Discounted Thompson Sampling for Non-Stationary Bandit Problems}
	
	\author{\name  Han Qi \email qihan19@stu.xjtu.edu.cn 
		\AND
		\name Yue Wang \email wy980521@stu.xjtu.edu.cn
	\AND
	\name Li Zhu \email zhuli@xjtu.edu.cn
	\\
		\addr 
		Xi'an Jiaotong University,\\
		School of Software Engineering,\\
		Xianning West Road 28,\\
		Xi'an,  710049, China
	}
	
	\maketitle

	\begin{abstract}
		Non-stationary multi-armed bandit (NS-MAB) problems have recently received significant attention. NS-MAB are typically modelled in two scenarios:  abruptly changing, where reward distributions remain constant for a certain period and  change at unknown time steps,
		and smoothly changing, where reward distributions evolve smoothly based on unknown dynamics.  In this paper, we propose Discounted Thompson Sampling (DS-TS) with Gaussian priors to address both non-stationary settings.  Our algorithm passively adapts to changes by incorporating a discounted factor into Thompson Sampling.  DS-TS method has been experimentally validated, but analysis of the regret upper bound is currently lacking. Under mild assumptions, we show that DS-TS with Gaussian priors can achieve nearly optimal regret bound on the order of  $\tilde{O}(\sqrt{TB_T})$ for abruptly changing and $\tilde{O}(T^{\beta})$ for smoothly changing, where $T$ is the number of time steps, $B_T$ is the number of breakpoints, $\beta$ is associated with the smoothly changing environment and $\tilde{O}$ hides the parameters independent of $T$ as well as logarithmic terms.
		Furthermore, empirical comparisons between  DS-TS  and other non-stationary bandit algorithms demonstrate its competitive performance. Specifically, when prior knowledge of the maximum expected reward is available, DS-TS has the potential to outperform state-of-the-art algorithms.
	\end{abstract}
	
	\section{Introduction}
	\label{Introduction}
	
	
	The multi-armed bandit (MAB) problem is a well-known sequential decision problem. In each time step, the learner must choose an arm (referred to as an action) from a finite set of arms based on previous observations. The learner only receives the reward associated with the chosen action and does not observe the rewards of other unselected actions. The learner's goal is to maximize the expected cumulative reward over time or, alternatively, to minimize the regret incurred during the learning process. Regret is defined as the difference between the expected reward of the optimal arm (the arm with the highest expected reward) and the expected reward achieved by the MAB algorithm being used. Minimizing regret implies approaching the performance of the optimal arm as closely as possible.

	
	MAB has found practical use in various scenarios, with one of the earliest applications being the diagnosis and treatment experiments proposed by \citeA{robbins1952some}. In this experiment, each patient's treatment plan corresponds to an arm in the MAB problem, and the goal is to minimize the patient's health loss by making optimal treatment decisions. Recently, MAB has gained wide-ranging applicability. For example, MAB algorithms have been used in online recommendation systems to improve user experiences and increase engagement \cite{li2011unbiased,bouneffouf2012contextual,li2016collaborative}.  Similarly, MAB has been employed in online advertising campaigns to optimize the allocation of resources and maximize the effectiveness of ad placements \cite{schwartz2017customer}.
	While the standard MAB model assumes fixed reward distributions, real-world scenarios often involve changing distributions over time. For instance, in online recommendation systems, the collected data gradually becomes outdated, and user preferences are likely to evolve \cite{wu2018learning}. This dynamic nature necessitates the development of algorithms that can adapt to these changes, leading to the exploration of non-stationary MAB problems.

	In recent years, extensive research has been conducted on non-stationary multi-armed bandit (MAB) problems. These research efforts can be broadly categorized into two approaches. The first category involves using change-point detection algorithms to identify when the reward distribution changes \cite{liu2018change,cao2019nearly,auer2019adaptively,chen2019new,besson2022efficient}. The second category focuses on passively reducing the influence of past observations \cite{garivier2011upper,raj2017taming,trovo2020sliding,baudry2021limited}.
	The former approach relies on certain assumptions about the changes in the distribution of arms to ensure the effectiveness of the change-point detection algorithm. For example, methods proposed by \citeA{liu2018change} and \citeA{cao2019nearly} require a lower bound on the amplitude of change in the expected rewards for each arm. The latter approach requires fewer assumptions about the characteristics of the changes. These methods often employ techniques such as sliding windows or discount factors to forget past information and adapt to the changing distribution of arms. Frequentist algorithms in both categories provide theoretical guarantees for regret upper bounds.
	However, in the case of Bayesian methods, such as Thompson Sampling, there is a lack of theoretical analysis regarding regret in non-stationary MAB problems, despite these algorithms demonstrating superior or comparable performance to frequentist algorithms in most non-stationary scenarios. To the best of our knowledge, only sliding window Thompson Sampling \cite{trovo2020sliding} has provided regret upper bounds. \citeA{raj2017taming} have explored discounted Thompson Sampling with Bernoulli priors but only derived the probability of selecting a sub-optimal arm in the simple case of a two-armed bandit.      
	      
	
	In this paper, we propose Discounted Thompson Sampling (DS-TS) with Gaussian priors for both abruptly changing and smoothly changing settings. In the former,  the distributions of rewards remain constant during a period of rounds and change at unknown rounds, while in the latter,  the reward distribution evolves smoothly based on unknown dynamics. 
	We adopt a unified method to analyze the regret upper bound for both non-stationary settings. We show that the regret upper bound of DS-TS for abruptly changing settings is $\tilde{O}(\sqrt{TB_T})$, where $T$ is the number of time steps, $B_T$ is the number of breakpoints. This regret bound matches the $\Omega(\sqrt{T})$ lower bound proven by \citeA{garivier2011upper} in an order sense. For the smoothly changing settings, we derive the regret bound of order $\tilde{O}(T^{\beta})$, where $\beta$ measures the number of rounds that the arms' expected rewards are close enough. 	In additional, we compare DS-TS with other non-stationary bandits algorithms empirically. Specially, if we know the information of the maximum of the expected rewards, by tuning the parameter $\tau_{max}$, our algorithm can outperform the state-of-the-art algorithms.

	\section{Related Work}
	
	Non-stationary MAB settings have received attention in the last few years. These methods can be roughly divided into two categories: they detect when the reward distribution changes with change-point detection algorithms or  they passively reduce the effect of past observations. Most of these works can achieve the regret upper bound of $\tilde{O}(\sqrt{T})$.
	
	Many works are based on the idea of forgetting past observations. Discounted UCB (DS-UCB) \cite{kocsis2006discounted,garivier2011upper} uses a discounted factor to average the past rewards. In order to achieve the purpose of forgetting information, the weight of the early reward is smaller. \citeA{garivier2011upper} also propose the sliding-window UCB (SW-UCB) by only using a few recent  rewards to compute the UCB index. They calculate the regret upper bound for DS-UCB and SW-UCB as $\tilde{O}(\sqrt{T B_T})$. EXP3.S, as proposed in \cite{auer2002nonstochastic}, has been shown to achieve the regret upper bound by $ \tilde{O}(\sqrt{TB_T})$. Under the assumption that the total variation of the expected rewards over the time horizon is bounded by a budget $V_T$, \citeA{besbes2014stochastic} introduce REXP3 with regret $\tilde{O}(T^{2/3})$.
	\citeA{combes2014unimodal} propose the SW-OSUB algorithm, specifically for the case of smoothly changing with an upper bound of $\tilde{O}(\sigma^{1/4}T)$, where $\sigma$ is the Lipschitz constant of the evolve process.
	\citeA{raj2017taming} propose the discounted Thompson sampling for Bernoulli priors without providing the regret upper bound. They only  calculate  the probability of picking a sub-optimal arm for the simple case of a two-armed bandit.
	Recently, \citeA{trovo2020sliding} propose the sliding-window Thompson sampling algorithm with regret $\tilde{O}(T^{\frac{1+\alpha}{2}})$ for abruptly changing and $\tilde{O}(T^{\beta})$ for smoothly changing. \citeA{baudry2021limited} propose a novel algorithm named  Sliding Window Last Block Subsampling Duelling Algorithm (SW-LB-SDA) with regret $\tilde{O}(\sqrt{TB_T})$. They only assume that the reward distributions belong to the same one-parameter exponential family for all arms during each stationary phase.
	
	There are also many works that exploit techniques  from the field of  change detection  to deal with reward distributions varying over time.
	\citeA{mellor2013thompson} combine a Bayesian change point mechanism and Thompson sampling strategy to  tackle the non-stationary problem. Their algorithm can detect  global switching and per-arm switching. \citeA{liu2018change} propose a change-detection framework that combines UCB and a change-detection algorithm named CUSUM. They obtain an upper bound for the average detection delay and a lower bound for the average time between false alarms. \citeA{cao2019nearly} propose M-UCB, which is similar to CUSUM but use another  simpler change-detection  algorithm.  M-UCB and CUMSUM are nearly optimal, their regret bounds are $\tilde{O}(\sqrt{TB_T})$.

	Recently, there are also some works deriving regret bounds without knowing the number of changes. For example, \citeA{auer2019adaptively} propose an algorithm called ADSWITCH with optimal regret bound $\tilde{O}(\sqrt{B_TT})$. \citeA{suk2022tracking} improve the work \cite{auer2019adaptively} so that the obtained regret bound is smaller than $\tilde{O}(\sqrt{ST})$, where $S$ only counts the best arms switches.
	
	
	\section{Problem Formulation}
	
	Assume that the non-stationary MAB problem has $K$ arms $\mathcal{A}:=\{ 1,2,...,K\}$ with finite time horizon  $T$. At round $t$, the learner must select an arm $ i_t \in \mathcal{A}$ and  obtain the corresponding reward $X_t(i_t)$. The rewards  are generated  from  different distributions (unknown to the learner) with  bounded  support. Without loss of generality, suppose the support set is $[0,1]$.
	The expectation of $X_t(i)$ is denoted as  $\mu_t(i) =\mathbb{E}[X_t(i)]$. 
	A policy $\pi$ is a function  $\pi(h_t)=i_t$ that selects  arm $i_t$ to play at round $t$.  
	Let $\mu_t(*):=\max_{i\in \{ 1,...,K\}} \mu_t(i)$ denote the expected reward of the optimal arm $i_t^{*}$ at round $t$. 
	Unlike the stationary MAB settings, where an arm is optimal all of the time (i.e. $\forall t \in \{1,...,T\}, i_t^{*}=i^{*}$), while in the non-stationary settings, the optimal arms might change over time.
	The  performance of a policy $\pi$ is measured in terms of cumulative expected regret:
	\begin{equation}
		\label{regret}
		R_T^{\pi}=\mathbb{E}\left[\sum_{t=1}^{T} (\mu_t(*)-\mu_{t}(i_t)) \right],
	\end{equation}
	
	\noindent where $\mathbb{E}[\cdot]$ is the expectation with respect to randomness of $\pi$. 
	Let $\Delta_t(i):=\mu_t(*)-\mu_t(i) $ and 
	let \[k_T(i) :=\sum_{t=1}^{T} \mathbbm{1}\{ i_t=i,i \neq i_t^{*} \}\] denote  the number of plays of arm $i $ when it is not the best arm until time $T$,
	\[ R_T^{\pi} = \sum_{i=1}^{K}\sum_{t=1}^{T}\Delta_t(i) \mathbb{E}[\mathbbm{1}\{i_t=i \}] \leq \sum_{i=1}^{K} \mathbb{E}[k_T(i)].\]
	When we analyze the upper bound of $R_T^{\pi}$, we can directly analyze $\mathbb{E}[k_T(i)]$ to get the upper bound of each arm. 
	Next, we give detailed description of the two non-stationary scenarios.
	
	\noindent {\bfseries{Abruptly Changing.}}
	The abruptly changing settings is  introduced by \citeA{garivier2011upper} for the first time.  
	Suppose the set of \textit{breakpoints} is $\mathcal{B}=\{b_1,...,b_{B_T}\}$ ( we define $b_1=1)
	$.  At each breakpoint, the reward distribution changes for at least one arm. The rounds between two adjacent breakpoints are called  \textit{stationary phase}. In the stationary phase, the reward distribution of all arms does not change.  Different from previous studies  \cite{besbes2014stochastic,liu2018change,cao2019nearly}, which imposed constraints on the  variation of expected rewards, we do not impose constraints on this variation in our settings.  \citeA{trovo2020sliding} makes assumption about the number of breakpoints to facilitate more generalized analysis, while we  explicitly use $B_T$ to represent the number of breakpoints for analysis.

	\noindent {\bfseries{Smoothly Changing.}}
	The smoothly changing setting have been studied by \citeA{combes2014unimodal,trovo2020sliding}. At each time step, the expected reward for each arm varies by no more than $\sigma$ and the learner doesn't have any information on how the rewards evolve. These limitations can be described by the following Lipschitz assumption:
	\begin{assumption}
		\label{assumption_sigma}
		There exits $\sigma>0$, for all $t,t^{'} \geq 1,1\leq i \leq K$, it holds that $| \mu_t(i)-\mu_{t^{'}}(i) | \leq \sigma |t-t^{'}|$.
		
	\end{assumption}

	\section{Discounted Thompson Sampling}
	
	In this section, we propose the Discounted Thompson Sampling algorithm with Gaussian priors for the non-stationary stochastic MAB problems. As in \cite{agrawal2013further}, our algorithm uses an implicit assumption that the likelihood of reward $X_i(t)$ can be modeled by Gaussian distribution. While the actual rewards distribution can be any distribution with support in $[0,1]$. 
	We use a discount factor $\gamma$ ($0<\gamma <1$) to dynamically adjust the estimate  of each arm's distribution. The key to our algorithm is to decrease the sampling variance of the selected arm  while increasing the sampling variance of the unselected arms.
	
	Specifically, let $N_t(\gamma,i)  =\sum_{j=1}^{t} \gamma^{t-j}  \mathbbm{1}\{ i_j=i \}$ denotes the discounted number of plays of arm $i$ until time $t$.
	We use $\hat{\mu}_t(\gamma,i) =\frac{1}{N_t(\gamma,i)}\sum_{j=1}^{t}  \gamma^{t-j}  X_j(i) \mathbbm{1}\{ i_j=i \}$ called discounted empirical average to estimate the expected rewards  of arm $i$. The sampling variance for arm $i$ at round $t$ is denoted as $ \tau_t(i)^2$.   At round $t$, arm $i$ samples from the Gaussian distribution $\mathcal{N}(\hat{\mu}_t(\gamma,i), \tau_t(i)^2 )$.  Recall that the rewards are in range $[0,1]$, thus the variance cannot  be increased to infinity, which would move the sampling away from the expectation. We restrict the upper bound on the variance to $\tau_{max}^2$, then the sampling variance is $\tau_t(i)=  \min\{ \frac{1}{\sqrt{N_t(\gamma,i)}},\tau_{max} \}$.

	Let $\tilde{\mu}_t(\gamma,i)= \sum_{j=1}^{t}  \gamma^{t-j}  X_j(i) \mathbbm{1}\{ i_j=i \}$ as the discounted cumulative reward. If arm $i$ is selected at round $t$,  the posterior distribution is updated as follows:
	\[ 
	\hat{\mu}_{t+1}(\gamma,i)= \frac{\gamma \hat{\mu}_t(\gamma,i)N_t(\gamma,i) + X_t(i)}{\gamma N_t(\gamma,i)+1}	=\frac{\tilde{\mu}_{t+1}(\gamma,i)}{N_{t+1}(\gamma,i)}
	\] 
	
	\noindent If arm $i$ isn't selected at round $t$, the posterior distribution is updated as 
	\[ 
	\hat{\mu}_{t+1}(\gamma,i)= \frac{\tilde{\mu}_{t+1}(\gamma,i)}{N_{t+1}(\gamma,i)} = \frac{\gamma \tilde{\mu}_{t}(\gamma,i)}{\gamma N_{t}(\gamma,i)}=\hat{\mu}_{t}(\gamma,i)
	\] 
	i.e. the expectation of posterior distribution remains unchanged.
	
	Algorithm \ref{algorithm1} shows the pseudocode of  DS-TS.
	We initialize the  prior distributions with $\hat{\mu}_1(i)=0,\tau_1(i)=1.$    
	Line 5 is the Thompson sampling. For each arm, we draw a random sample $\theta_t(i)$ from $\mathcal{N}(\hat{\mu}_t(\gamma,i),\tau_t(i)^2)$. Then we select  arm $i_t $ with the maximum sample value to play and obtain the reward $X_{t}(i_t)$ (Line 7).  To avoid the time complexity  going to $O(T^2)$, we introduce $\tilde{\mu}_t(i)$ to calculate $\hat{\mu}_t(\gamma,i)$ using an iterative method(Line 9-11). 
	Finally, we update the posterior variance $\tau_i^2$ of each arm  (Line 12).

	\begin{algorithm}[!htbp]
		\label{algorithm1}
		\caption{DS-TS }
		{\bfseries Input:} discounted factor $\gamma \in (1-\frac{1}{e},1)$, $\tau_{max}$,\\ \quad \quad \quad  $\hat{\mu}_1(i) = 0$, $\tilde{\mu}_1(i)=0$, $N_t(\gamma,i)=0$,  $\tau_1(i)=\tau_{max}.$\\
		\For{$t=1,...,T$}{
			\For {$i=1,..,K$}{
				sample $\theta_t(i)$ independently from  $\mathcal{N}(\hat{\mu}_t(\gamma,i),\tau_t(i)^2)$ \\
			}
			Play arm $i_t = \arg\max_i\theta_t(i) $  and observe reward $X_t(i_t)$.\\
			
			\For {$i=1,...,K$}{
				$\tilde{\mu}_{t+1}(\gamma,i)=\gamma \tilde{\mu}_t(\gamma,i) + \mathbbm{1}\{i=i_t\} X_t(i_t) $\\
				$N_{t+1}(\gamma,i)=\gamma N_t(\gamma,i)+\mathbbm{1}\{i=i_t\}$\\
				$\hat{\mu}_{t+1}(\gamma,i)= \frac{\tilde{\mu}_{t+1}(\gamma,i)}{N_{t+1}(\gamma,i)} $\\
				$\tau_{t+1}(i)=  \min\{ \frac{1}{\sqrt{N_{t+1}(\gamma,i)}},\tau_{max} \}$	
			}
		}
	\end{algorithm}

	\noindent {\bfseries{Related to Thompson Sampling. }}
	If $\gamma=1$, DS-TS is equivalent to Thompson Sampling with stationary settings proposed by \citeA{agrawal2013further} except for the variance. Their sampling variance is $\frac{1}{k_i+1} $($k_i$ is the number  plays of arm $i$), which are updated according to the standard Bayesian posterior distribution. While we truncate it to $\tau_{max}$  to prevent the posterior variance from becoming infinite.

	\noindent {\bfseries{Related to Discounted UCB. }}
	Line 7 in Algorithm \ref{algorithm1} can be rewritten as 
	\[ i_t=\arg \max_i \hat{\mu}_t(\gamma,i) + \epsilon_t(i), \epsilon_t(i) \sim \mathcal{N}(0,\tau_t(i)^2)  \]
	We use the same method as DS-UCB \cite{garivier2011upper} to update $\hat{\mu}_t(\gamma,i)$. As for selecting the best arm,  DS-UCB uses a padding function $c_t(\gamma,i)=2B\sqrt{\frac{\xi\log n_t(\gamma)}{N_t(\gamma,i)}}$,  where $B,\xi$ is the tuning parameters, $n_t(\gamma)=\sum_{i=1}^{K}N_t(\gamma,i)$.  
	While our approach can be viewed as replacing the padding function with $\epsilon_t(i)$.
	Sampling from the normal distribution $\mathcal{N}(0,\tau_t(i)^2) $ is more exploratory than using deterministic padding function. In the experiment section, we will show the advantages of DS-TS.

	\section{Our Results}
	
	In this section, we  give the upper bounds of DS-TS with abruptly changing and smoothly changing settings. Then we discuss how to take the values of the parameters so that the DS-TS reaches the optimal upper bound.
	\subsection{Abruptly Changing Settings}
	Recall that  $\Delta_t(i):=\mu_t(*)-\mu_t(i) $.
	Let $\Delta_T=\min\{ \Delta_t(i):t\leq T, i \neq i_t^* \}$, be the minimum  difference between the expected reward of the best arm $i_t^*$ and the expected reward of all arm  in all time $T$ when the arm  is not the best arm, and  $\mu_{max}= \max_{t \in \{1,...,T\},i \in \{1,...,K\}}\mu_i(t) \in (0,1] $, be the maximum of  expected rewards.
	Define function $F(x)=\frac{1}{\sqrt{2\pi}}\frac{x}{1+x^2}e^{-x^2/2}.$
	\begin{theorem}
		\label{result1}
		Let $\gamma \in (1-\frac{1}{e},1), \tau_{max}\geq\frac{1}{12\sqrt{2}}$. In the abruptly changing settings, for any arm $i \in \{1,...,K\}$,
			\begin{equation*}
				\mathbb{E}[k_T(i)]  \leq B_T D(\gamma)  + (C+2)L(\gamma)\gamma^{-1/(1-\gamma)} T(1-\gamma)\log(\frac{1}{1-\gamma}), 
			\end{equation*}
	\end{theorem}
	\noindent where $D(\gamma)=\frac{\log(-(1-\gamma)^2\log(1-\gamma))}{\log\gamma}$,$L(\gamma)= \frac{144(1+\sqrt{2})^2\log(\frac{1}{1-\gamma}+e^{25})}{\gamma^{1/(1-\gamma)}(\Delta_T)^2},C=e^{25}+12+\frac{1}{F(\frac{\mu_{max}}{\tau_{max}})}  $.\\

	\begin{corollary}
		\label{corollary1}
		When $\gamma$ is close to $1$,  $\gamma^{-\frac{1}{1-\gamma}}$ is around $e$. If the time horizon $T$ and number of breakpoints $B_T$ are known in advance, the discounted factor can be chosen as $\gamma =1-\sqrt{B_T/T}$, then $\mathbb{E}[k_T(i)] = O(\sqrt{TB_T}\log^2(T) )$. If $B_T = O(T^{\alpha})$ for some $\alpha \in (0,1)$, this regret is upper bounded by $O(T^{\frac{1+\alpha}{2}}\log^2(T) )$.

	\end{corollary}

	\begin{remark}
		The condition $\tau_{max}\geq \frac{1}{12\sqrt{2}}$ is imposed to help the analysis. From the expression for $F(x)$ it is obvious that $\tau_{max}$ has a lower bound greater than 0.
		In fact, follows from the proof of Lemma \ref{new}, $\tau_{max}$ needs to satisfy the condition $\tau_{max} \geq \frac{\Delta_T}{12\sqrt{2}+3\sqrt{1-\gamma}}\frac{1}{\sqrt{\log\frac{1}{1-\gamma}}} $. Since $\gamma >1-\frac{1}{e}$,  $ \tau_{max} \geq \frac{1}{12\sqrt{2}}$  is clearly satisfied.
		Combining practical experience, too small sampling variance will make Thompson Sampling lose its exploration ability. 
		In addition, if we know $\mu_{max}$ in advance,  $\tau_{max}$ can be set suitably  to improve the  empirical performance of DS-TS. In the experimental section we will describe how to take the appropriate $\tau_{max}$.
	\end{remark}
	
	\subsection{Smoothly Changing Settings}
		
	Smoothly changing settings present greater challenges compared to abruptly changing settings. The primary difficulty arises from the fact that the expected rewards of different arms may be extremely close to each other. In order to effectively address this challenge, we require certain assumption that restricts the number of rounds in which the expected rewards of two arms can become arbitrarily small. This assumption is necessary to ensure that the arms' expected rewards remain distinguishable and prevent the algorithm from being overwhelmed by the inherent uncertainty in the rewards.
	Let $\Delta \in (0,1)$, define 
	\[  H(\Delta,T)= \{   t \in \{1,...,T\}: \exists i\neq j,| \mu_i(t)-\mu_j(t)|< \Delta \} \]
	and we make the following assumption.
	\begin{assumption}
		\label{assumption_delta}
		There exist some constant independent of $T$, $\beta \in [0,1]$, positive number $F$, $\Delta_0 \in (0,1)$, s.t. for all $\Delta <\Delta_0$,
		\[  | H(\Delta,T)| \leq F\Delta T^{\beta}  \]
	\end{assumption}
	This assumption  is consistent with  \cite{trovo2020sliding}, and if $\beta=1$  it's equivalent to that in \cite{combes2014unimodal}. 
	\begin{theorem}
		\label{result2}
		Let $\gamma \in (1-\frac{1}{e},1), \tau_{max}\geq\frac{1}{12\sqrt{2}}$ Lipschitz constant $\sigma>0$. There exists $\Delta_0  $ as in Assumption \ref{assumption_delta}, s.t. $2\sigma D(\gamma) < \Delta/3\leq \Delta_0$. In the smoothly changing settings, for any arm $i \in \{1,...,K\}$,
	\begin{equation}
	\mathbb{E}[k_T(i)]  \leq F \Delta T^{\beta}  +  M(\gamma)T(1-\gamma)\log(\frac{1}{1-\gamma}), 
	\end{equation}
	where $M(\gamma)= \frac{144(1+\sqrt{2})^2\log(\frac{1}{1-\gamma}+e^{25})}{\gamma^{1/(1-\gamma)}\Delta^2}(e^{25}+13+\frac{1}{F(\frac{\mu_{max}}{\tau_{max}})})  + \frac{594}{\gamma^{1/(1-\gamma)}(\Delta/3-2\sigma D(\gamma))^2} .$
	\end{theorem}
	
	\begin{corollary}
		\label{corollary2}
		If $\beta$ in Assumption \ref{assumption_delta} and the time horizon $T$ are known in advance, the discounted factor can be set as $\gamma=1-1/T^{1-\beta}$, then  $\mathbb{E}[k_T(i)] = O(T^{\beta}\log^2(T) )$. 
	\end{corollary}
	\begin{remark}
		\citeA{trovo2020sliding}  discusses in detail the value of $\beta$ for which Assumption \ref{assumption_delta} holds under the conditions of Theorem \ref{result2} and Corollary \ref{corollary2}. Define 
		\[ P=|\{ t \in \{1,...,T-1\}: \exists i\neq j,(\mu_i(t)-\mu_j(t))(\mu_i(t+1)-\mu_j(t+1))<0  \}| \]
		as the number of times the expected rewards of a pair of arms change over the time period.  When $\gamma=1-1/T^{1-\beta}$, it's easy to get $D(\gamma) \leq 2T^{1-\beta}\log T$.
		We can get  $\beta$ need within $ [ \max \{1-\log_T (\frac{\Delta}{12\sigma \log T}), \frac{1}{2}-\log_T\sqrt{\frac{F\Delta}{24P\log T}} \},1 ]$ in our setting. 
	\end{remark}
	\section{Proofs of Upper Bounds}
	
	In this section, we prove the results  respectively. The proofs of the two settings follow a similar approach. The differences in the proof can be dealt with uniformly using Lemma \ref{D} and Lemma \ref{D1}. The core framework of our proof follows \cite{agrawal2013further}. We  extend the analysis of \cite{agrawal2013further} to the non-stationary case by combining the method proposed in  \cite{garivier2011upper}. The main difficulty lies in the need to estimate the  $\mathbb{E}[\frac{1}{p_{i,t}}]$ for non-stationary settings (Lemma \ref{new} and Lemma \ref{new1}).

	\subsection{Proofs of Theorem \ref{result1}}
	{\bfseries{Proof outlines:}} 
	 The  idea is to divide 
	$\mathbb{E}[k_T(i)]$ into several parts according to whether the specific events  are true or not.
	First of all the expected rewards for all arms will not be well estimated in the rounds near the  breakpoints, and this part can be bounded as $B_TD(\gamma)$. Then focus on the rounds that are far from the breakpoints($D(\gamma)$ rounds after the breakpoints).
	For the rounds in which the mean of the arm is not well estimated, the regret can be bounded by  Lemma \ref{N<A}. For the rounds that the arm is fully  explored, the regret bound can be estimated by self-normalized Hoeffding-type inequality \cite{garivier2011upper}. For the other cases, one can use Lemma \ref{pit} and thus needs to estimate $\mathbb{E}[\frac{1}{p_{i,t}}]$. We derive the upper bound of $\mathbb{E}[\frac{1}{p_{i,t}}]$ for non-stationary settings, with an extra logarithmic term compared with the stationary settings. 
	
	Before proceeding to the specific analysis, we first give some definitions and lemmas that will be used in both non-stationary settings. 
	\begin{definition}[\bfseries{Quantities} $x_t(i),y_t(i)$]
		For  arm $i \neq i_t^{*}$, we choose two threshold $x_t(i),y_t(i)$ such that $x_t(i)=\mu_t(i)+\frac{\Delta_t(i)}{3},y_t(i)=\mu_t(*)- \frac{\Delta_t(i)}{3}  $. Then $ \mu_t(i)<x_t(i)<y_t(i)<\mu_t(*)$ and $y_t(i)-x_t(i)= \frac{\Delta_t(i)}{3} $.
		
	\end{definition}
	
	\begin{definition}[$\ddot{\mu}_t(\gamma,i)$]
		$\ddot{\mu}_t(\gamma,i)=\frac{1}{N_t(\gamma,i)}\sum_{j=1}^{t}\gamma^{t-j}\mathbbm{1}\{ i_j=i \}\mu_j(i)$ denotes the discounted average of expectation for arm $i$ at time step $t$.  If the randomness of $\mathbbm{1}\{i_j=i\}$ is removed, i.e., the selection of each arm is deterministic, then $\ddot{\mu}_t(\gamma,i)=\mathbb{E}[\hat{\mu}_t(\gamma,i)] $.
	\end{definition}
	
	\begin{definition}[\bfseries{History} $\mathcal{F}_t$] 
		$\mathcal{F}_t$ is the squence 
		\[ \mathcal{F}_t = \{ i_k,X_k(i_k),k=1,...,t \}, \]
		where $i_k$ denotes the arm played at time $k$, and $X_k(i_k)$ denotes the reward obtained at time $k$. Define $\mathcal{F}_0=\{ \}$. 
	\end{definition}

	By definition, we know $\mathcal{F}_0 \subseteq \mathcal{F} \subseteq ... \subseteq \mathcal{F}_{T}$. And $i_t,\hat{\mu}_t(\gamma,i)$, the distribution of $\theta_t(i)$  are determined by the history $\mathcal{F}_{t-1}$.
	
	Now we give some additional definitions and lemmas for abruptly changing settings only.
	The abruptly changing settings is in fact piecewise-stationary. Some rounds between  two breakpoints  appear to be stationary. Based on this observation, we give the following definition.
	\begin{definition}[\bfseries{Pseudo-Stationary Phase} $\mathcal{T}(\gamma)$]
		\label{pseudo}
		$\mathcal{T}(\gamma) = \{ t \leq T: \forall s \in (t-D(\gamma),t],\mu_s(\cdot)=\mu_t(\cdot)  \}$, where
		$D(\gamma)=\log((1-\gamma)^2\log( \frac{1}{1-\gamma}) )/\log\gamma$.
	\end{definition}
	
	\begin{remark}
		\label{|S|}
		Let $\mathcal{S}(\gamma)=\{t\leq T: t \notin \mathcal{T}(\gamma) \}$. Note that, on the right side of any breakpoint, there will be $D(\gamma)$ rounds belonging to $\mathcal{S}(\gamma)$. Therefore, the number of elements in the set $\mathcal{S}(\gamma)$ has an upper bound $B_T D(\gamma)$, i.e. $|\mathcal{S}(\gamma)| \leq B_T D(\gamma)$. $\mathcal{T}(\gamma)$ is called the pseudo-stationary phase because the length of $\mathcal{T}(\gamma)$ is smaller than the true stationary phase. 
		Figure \ref{fig_DS} shows $\mathcal{T}(\gamma)$ and $\mathcal{S}(\gamma)$ in two different situations.
	\end{remark}
	
	\begin{figure}[!htbp] 
		
		\centering{\includegraphics[height=4cm,width=7cm]{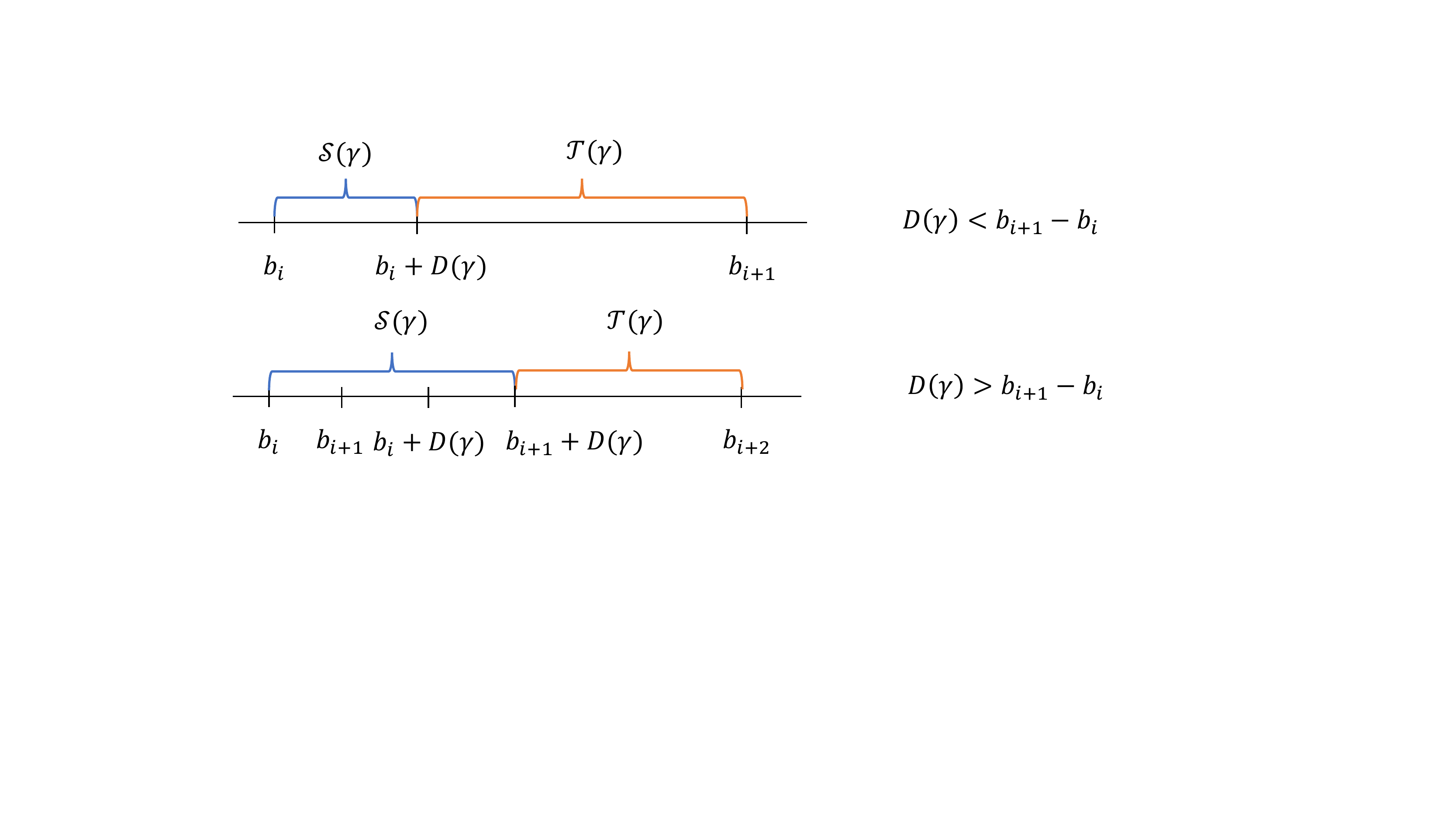}}
		\caption{ Illustration of $\mathcal{T}(\gamma)$ and $\mathcal{S}(\gamma)$ in two different situations. $b_{i+1}-b_i>D(\gamma)$ is shown in the top figure, and $ b_{i+1}-b_i \leq D(\gamma)$ the bottom figure.} 
		\label{fig_DS}
	\end{figure}

	Unlike the Sliding window method\cite{garivier2011upper,trovo2020sliding}, $\mu_t(i)$ and $\ddot{\mu}_t(\gamma,i)$ are often not the same even in the pseudo-stable phase $\mathcal{T}(\gamma)$.
	The following lemma depicts that after finite rounds at the breakpoint, that is, in the pseudo-stable phase, the distance between $\mu_t(i)$ and $\ddot{\mu}_t(\gamma,i)$ has an upper bound.
	\begin{lemma}
		\label{D}
		$\forall t \in \mathcal{T}(\gamma)$,  the distance between $\mu_t(i)$ and $\ddot{\mu}_t(\gamma,i)$ is less than $U_t(\gamma,i)$.
		\begin{equation}
			\label{U}
			| \mu_t(i) - \ddot{\mu}_t(\gamma,i| \leq U_t(\gamma,i), 
		\end{equation}
		
		where 
		\[U_t(\gamma,i)= \sqrt{\frac{(1-\gamma)\log \frac{1}{1-\gamma}}{N_t(\gamma,i)}}.
		\]
		
	\end{lemma}

	As can be seen from Definition \ref{pseudo}, $(t-D(\gamma),t) \text{ does not contain any breakpoints} \Longleftrightarrow  t \in \mathcal{T}(\gamma)$.
	For any breakpoint $b_i \in \{ b_1,...,b_{B_T} \}$,  $b_i +D(\gamma) \in \mathcal{T}(\gamma)$ if $D(\gamma) \leq b_{i+1}-b_{i}$. 
	That is, $D(\gamma)$ rounds after the breakpoint $b_i$ ($D(\gamma) \leq b_{i+1}-b_{i}$), the distance between $\mu_t(i)$ and $\ddot{\mu}_t(\gamma,i)$ is less than $U_t(\gamma,i)$. 
	
	\begin{lemma}
		\label{new}
		Let $p_{i,t}=\mathbb{P}(\theta_t(*)>y_t(i)|\mathcal{F}_{t-1})$. For any $t \in \mathcal{T}(\gamma)$ and $i \neq i_t^{*}$,
		\[
		 \sum_{t \in \mathcal{T}(\gamma)} \mathbb{E}[\frac{1-p_{i,t}}{p_{i,t}}\mathbbm{1}\{ i_t=i_t^{*}, \theta_t(i)<y_t(i) \} ] \leq CT(1-\gamma)L(\gamma)\gamma^{-1/(1-\gamma)}\log\frac{1}{1-\gamma}.
		\]
	\end{lemma}
	where $C=e^{25}+\frac{1}{F(\frac{\mu_{max}}{\tau_{max}})} +12$, $L(\gamma) =\frac{144(1+\sqrt{2})^2\log(\frac{1}{1-\gamma}+e^{25})}{\Delta_T^2}. $

	\noindent To facilitate the analysis, we  define some quantities that are independent of $t$.  \[m=\frac{12\sqrt{2}}{\sqrt{1-\gamma}}+3,n=12\sqrt{2}+3\sqrt{1-\gamma}, A(\gamma)=\frac{n^2\log(\frac{1}{1-\gamma})}{(\Delta_T)^2}, \]
	From the definition of $U_t(\gamma,i)$ given in Lemma \ref{D}, we can get 
	\begin{equation}
		\label{UA}
		U_t(\gamma,i)=\frac{\Delta_T}{m}\sqrt{\frac{A(\gamma)}{N_t(\gamma,i)}}.
	\end{equation}

	\noindent Now we can give the detailed proof. The proof is in $5$ steps:
	
	\noindent {\bfseries{Step 1}} 
	We can divide the rounds $t \in \{1,...,T\}$ into two parts: $ \{ t \in \mathcal{T}(\gamma)\}$ and $ \{ t \notin \mathcal{T}(\gamma)\}$. From Remark \ref{|S|}, the number of elements in the second part is smaller than $B_TD(\gamma)$.
	\begin{equation}
			\mathbb{E}[k_T(i)] \leq B_TD(\gamma) +\sum_{t \in \mathcal{T}(\gamma)}\mathbb{P}(i_t=i).
	\end{equation}

	\noindent {\bfseries{Step 2}} Then we consider the event $\{ N_t(\gamma,i)>A(\gamma)\}$.
	\[
	\sum_{t \in \mathcal{T}(\gamma)}\mathbb{P}(i_t=i)=\sum_{t \in\mathcal{T}(\gamma)}\mathbb{P}(i_t=i,N_t(\gamma,i)<A(\gamma))+\sum_{t \in \mathcal{T}(\gamma)}\mathbb{P}(i_t=i,N_t(\gamma,i)>A(\gamma)).
	\]

	\noindent We first bound $\sum_{t \in \mathcal{T}(\gamma)}\mathbb{P}(i_t=i,N_t(\gamma,i)<A(\gamma))$.
	\begin{equation}
		\label{EE}
		\begin{aligned}
			\sum_{t \in \mathcal{T}(\gamma)}\mathbb{P}(i_t=i,N_t(\gamma,i)<A(\gamma))
			&=\sum_{t \in \mathcal{T}(\gamma)}\mathbb{E}[\mathbb{P}(i_t=i,N_t(\gamma,i)<A(\gamma)|\mathcal{F}_{t-1})]\\
			&= \sum_{t \in \mathcal{T}(\gamma)}\mathbb{E}[\mathbb{E}[\mathbbm{1} \{i_t=i,N_t(\gamma,i)<A(\gamma)\}|\mathcal{F}_{t-1}]]\\
			& \stackrel{(a)}{=}\sum_{t \in \mathcal{T}(\gamma)}\mathbb{E}[\mathbbm{1} \{i_t=i,N_t(\gamma,i)<A(\gamma)\}] \\
			& \stackrel{(b)}{\leq} T(1-\gamma) A(\gamma) \gamma^{-1/(1-\gamma)}
		\end{aligned}
	\end{equation}
	where (a) uses  the tower rule for expectation, (b) follows from Lemma \ref{N<A}.
	Therefore,
	\begin{equation}
		\label{tmp}
			\mathbb{E}[k_T(i)] \leq T(1-\gamma) A(\gamma) \gamma^{-1/(1-\gamma)} + B_T D(\gamma) 
		+ \sum_{t \in \mathcal{T}(\gamma)}\mathbb{P}(i_t=i,N_t(\gamma,i)>A(\gamma))
	\end{equation}
	\noindent {\bfseries{Step 3}} 
	Define $E_t(\gamma,i) $ as the event that $\{ i_t=i,N_t(\gamma,i)>A(\gamma) \}$. Define $E_t^{\theta}(i)$ as the event that $\theta_t(i)<y_t(i)$. 
	This part may be decomposed as follows:
	\begin{equation}
		\label{tmp2}
		\begin{aligned}
		\sum_{t \in\mathcal{T}(\gamma)}\mathbb{P}(E_t(\gamma,i)) &= \sum_{t \in\mathcal{T}(\gamma)} \mathbb{P}(E_t(\gamma,i),\hat{\mu}_t(\gamma,i)>x_t(i)) 
			+ \sum_{t \in\mathcal{T}(\gamma)} \mathbb{P}(E_t(\gamma,i), \hat{\mu}_t(\gamma,i)<x_t(i), \overline{E_t^{\theta}(i)}  )\\
			&+ \sum_{t \in\mathcal{T}(\gamma)} \mathbb{P}(E_t(\gamma,i), \hat{\mu}_t(\gamma,i)<x_t(i) ,E_t^{\theta}(i))
		\end{aligned}
	\end{equation}
	Next, we bound the first part by Lemma \ref{hh}.
	\begin{equation}
		\label{firstpart}
		\begin{aligned}
			 &\sum_{t \in\mathcal{T}(\gamma)} \mathbb{P}(E_t(\gamma,i),\hat{\mu}_t(\gamma,i)>x_t(i)) \\
			&\leq \sum_{t \in\mathcal{T}(\gamma)} \mathbb{P}(\hat{\mu}_t(\gamma,i)>\mu_t(i)+\frac{\Delta_t(i)}{3},N_t(\gamma,i)>A(\gamma))\\
			&\leq T (1-\gamma)^{48}\log \frac{1}{1-\gamma}
		\end{aligned}
	\end{equation}
	
	\noindent {\bfseries{Step 4}} Then we bound the second part.
	
	\begin{equation}
		\label{secondpart}
		\begin{aligned}
		&\sum_{t \in\mathcal{T}(\gamma)} \mathbb{P}(E_t(\gamma,i), \hat{\mu}_t(\gamma,i)<x_t(i), \overline{E_t^{\theta}(i)} )\\
		&=\mathbb{E}[ \sum_{t \in\mathcal{T}(\gamma)} \mathbb{E}[ \mathbbm{1}\{ i_t=i,N_t(\gamma,i)>A(\gamma), \hat{\mu}_t(\gamma,i)<x_t(i), \overline{E_t^{\theta}(i)}  \}| \mathcal{F}_{t-1} ]  ]\\
		& \stackrel{(c)}{=}  \mathbb{E}[ \sum_{t \in\mathcal{T}(\gamma)} \mathbbm{1}\{ N_t(\gamma,i)>A(\gamma), \hat{\mu}_t(\gamma,i)<x_t(i)\} \mathbb{P}(i_t=i,\overline{E_t^{\theta}(i)} | \mathcal{F}_{t-1})  ]\\
		&\leq \mathbb{E}[ \sum_{t \in\mathcal{T}(\gamma)} \mathbbm{1}\{ N_t(\gamma,i)>A(\gamma), \hat{\mu}_t(\gamma,i)<x_t(i)\} \mathbb{P}(\theta_t(i)>y_t(i) | \mathcal{F}_{t-1}) ],
	\end{aligned}
\end{equation}
	\noindent where (c) uses the fact that $N_t(\gamma,i)$ and $\hat{\mu}_t(i)$ are determined by the history $\mathcal{F}_{t-1}.$
	Therefore, given the history $\mathcal{F}_{t-1}$ such that $ N_t(\gamma,i)>A(\gamma)$ and $ \hat{\mu}_t(\gamma,i)<x_t(i)$, we have 
	\[  
	\mathbb{P}( \theta_t(i)>y_t(i) | \mathcal{F}_{t-1}) ) \leq \mathbb{P}( \theta_t(i)- \hat{\mu}_t(\gamma,i) >\frac{\Delta_T}{3} | \mathcal{F}_{t-1}) ) \leq \frac{1}{2}\exp (-\frac{(\Delta_T)^2A(\gamma)}{18} )\leq \frac{1}{2}(1-\gamma)^{16}.
	 \]
	 For other $\mathcal{F}_{t-1}$, the  indicator term $\mathbbm{1}\{ N_t(\gamma,i)>A(\gamma), \hat{\mu}_t(\gamma,i)<x_t(i)\}$ will be $0$. Hence, we can bound the second part by
	 \[
	 \sum_{t \in\mathcal{T}(\gamma)} \mathbb{P}(E_t(\gamma,i), \hat{\mu}_t(\gamma,i)<x_t(i), \overline{E_t^{\theta}(i)} ) \leq \frac{T}{2}(1-\gamma)^{16}
	 \]
	
	\noindent {\bfseries{Step 5}} Finally, 
	using Lemma \ref{new},
	\[\begin{aligned} 
	\sum_{t \in\mathcal{T}(\gamma)} \mathbb{P}(E_t(\gamma,i), \hat{\mu}_t(\gamma,i)<x_t(i) ,E_t^{\theta}(i)) 
	&\leq \sum_{t \in\mathcal{T}(\gamma)} \mathbb{E}[ \frac{1-p_{i,t}}{p_{i,t}}\mathbb{P}( i_t=i_t^{*}, E_t^{\theta}(i)|\mathcal{F}_{t-1} ) ] \\
	&\stackrel{(d)}{=} \sum_{t \in\mathcal{T}(\gamma)} \mathbb{E}[\mathbb{E}[ \frac{1-p_{i,t}}{p_{i,t}}\mathbbm{1}\{ i_t=i_t^{*}, E_t^{\theta}(i) \}|\mathcal{F}_{t-1}  ]] \\
	&= \sum_{t \in\mathcal{T}(\gamma)} \mathbb{E}[ \frac{1-p_{i,t}}{p_{i,t}}\mathbbm{1}\{ i_t=i_t^{*}, E_t^{\theta}(i) \}  ]\\
	&\leq CT(1-\gamma)L(\gamma)\gamma^{-1/(1-\gamma)}\log\frac{1}{1-\gamma}.
	\end{aligned}\]
	where (d)  uses the fact that $p_{i,t}$ is fixed given $\mathcal{F}_{t-1}$, 
	
	Substituting the results in Step 3-5 to Equation \eqref {tmp2} and Equation \eqref {tmp}, 
	\[
	\begin{aligned}
		\mathbb{E}[k_T(i)] &\leq T(1-\gamma) A(\gamma) \gamma^{-1/(1-\gamma)} + B_T D(\gamma) 
		+ \sum_{t \in \mathcal{T}(\gamma)}\mathbb{P}(i_t=i,N_t(\gamma,i)>A(\gamma))\\
		&\leq T(1-\gamma) A(\gamma) \gamma^{-1/(1-\gamma)} + B_T D(\gamma)  + T(1-\gamma)\log\frac{1}{1-\gamma}\\&+ CT(1-\gamma)L(\gamma)\gamma^{-1/(1-\gamma)}\log\frac{1}{1-\gamma}\\
		&\leq B_TD(\gamma) + (C+2)L(\gamma)\gamma^{-1/(1-\gamma)}T(1-\gamma)\log\frac{1}{1-\gamma}.
	\end{aligned}
	\]

	\subsection{Proofs of Theorem \ref{result2}}
	The proof of Theorem \ref{result2} is similar to  Theorem \ref{result1}. The main difference is that there is no pseudo-stationary phase under smoothly changing settings.
	Fortunately, conclusions (Lemma \ref{D1},Lemma \ref{new1}) similar to Lemma \ref{D} and Lemma \ref{new} still hold.
	
	Recall that
	\[U_t(\gamma,i)=\sqrt{\frac{(1-\gamma)\log \frac{1}{1-\gamma}}{N_t(\gamma,i)}}, D(\gamma)=\frac{\log((1-\gamma)^2\log( \frac{1}{1-\gamma}) )}{\log\gamma},\]\[ m=\frac{12\sqrt{2}}{\sqrt{1-\gamma}}+3,n=12\sqrt{2}+3\sqrt{1-\gamma}\]
	
	\begin{lemma}
		\label{D1}
		For any $t$ and $\sigma$ satisfies Assumption \ref{assumption_sigma},
		\begin{equation}
			\label{U}
			| \mu_t(i) - \ddot{\mu}_t(\gamma,i| \leq U_t(\gamma,i)+\sigma D(\gamma), 
		\end{equation}
	\end{lemma}
	
	\begin{lemma}
		\label{new1}
		Let $p_{i,t}=\mathbb{P}(\theta_t(*)>y_t(i)-\sigma D(\gamma)|\mathcal{F}_{t-1})$. For any  $i \neq i_t^{*}$, 
		\[
		\sum_{t =1}^{T} \mathbb{E}[\frac{1-p_{i,t}}{p_{i,t}}\mathbbm{1}\{ i_t=i_t^{*}, \theta_t(i)<y_t(i)-\sigma D(\gamma) \} ] \leq CT(1-\gamma)L(\gamma)\gamma^{-1/(1-\gamma)}\log\frac{1}{1-\gamma}.
		\]
	\end{lemma}
	where $C=e^{25}+\frac{1}{F(\frac{\mu_{max}}{\tau_{max}})} +12$, $L(\gamma) =\frac{144(1+\sqrt{2})^2\log(\frac{1}{1-\gamma}+e^{25})}{\Delta^2}. $ 
	
	We redefine $A(\gamma)$ in smoothly changing settings as
	\[ A(\gamma)=\frac{n^2\log(\frac{1}{1-\gamma})}{(\Delta/3-2\sigma D(\gamma))^2}. \]
	
	Since there is no pseudo-stationary phase, the proof is only need divide into four steps. 

	\noindent {\bfseries {Step 1}}
	The first step is exactly the same as the step 2 of  Theorem \ref{result1}, so we can get the folowing directly:
	\begin{equation}
		\label{tmp3}
		\mathbb{E}[k_T(i)] \leq F\Delta T^{\beta}+ T(1-\gamma) A(\gamma) \gamma^{-1/(1-\gamma)} 
		+ \sum_{t=1}^{T}\mathbb{P}(i_t=i,N_t(\gamma,i)>A(\gamma))
	\end{equation}
	
	\noindent {\bfseries {Step 2}} Define $E_t(\gamma,i) $ as the event that $\{ i_t=i,N_t(\gamma,i)>A(\gamma) \}$. Then 
	\begin{equation}
		\begin{aligned}
			\sum_{t =1}^{T}\mathbb{P}(E_t(\gamma,i)) &= \sum_{t =1}^{T} \mathbb{P}(E_t(\gamma,i),\hat{\mu}_t(\gamma,i)>x_t(i)+\sigma D(\gamma)) \\
			&+ \sum_{t=1}^{T} \mathbb{P}(E_t(\gamma,i), \hat{\mu}_t(\gamma,i)<x_t(i)+\sigma D(\gamma), \theta_t(i)>y_t(i)-\sigma D(\gamma)  )\\
			&+ \sum_{t =1}^{T} \mathbb{P}(E_t(\gamma,i), \hat{\mu}_t(\gamma,i)<x_t(i)+\sigma D(\gamma), \theta_t(i)<y_t(i)-\sigma D(\gamma) )
		\end{aligned}
	\end{equation}
	The first part can be bounded by Lemma \ref{hh},
	\[ \sum_{t =1}^{T} \mathbb{P}(E_t(\gamma,i),\hat{\mu}_t(\gamma,i)>x_t(i)+\sigma D(\gamma)) \leq T(1-\gamma)^{48} \log \frac{1}{1-\gamma}\]
	
	\noindent {\bfseries {Step 3}} The second part can be bounded through a similar method as Equation \eqref {secondpart}.
	\begin{equation}
		\begin{aligned}
			&\sum_{t=1}^{T} \mathbb{P}(E_t(\gamma,i), \hat{\mu}_t(\gamma,i)<x_t(i)+\sigma D(\gamma), \theta_t(i)>y_t(i)-\sigma D(\gamma)  )\\
			&\leq \mathbb{E}[ \sum_{t =1}^{T} \mathbbm{1}\{ N_t(\gamma,i)>A(\gamma), \hat{\mu}_t(\gamma,i)<x_t(i)+\sigma D(\gamma)\} \mathbb{P}(\theta_t(i)>y_t(i)-\sigma D(\gamma) | \mathcal{F}_{t-1}) ]
		\end{aligned}
	\end{equation}
	Given the history $\mathcal{F}_{t-1}$ such that $ N_t(\gamma,i)>A(\gamma)$ and $ \hat{\mu}_t(\gamma,i)<x_t(i)+\sigma D(\gamma)$, we have 
	\[  
	\begin{aligned}
	\mathbb{P}( \theta_t(i)>y_t(i)-\sigma D(\gamma) | \mathcal{F}_{t-1}) ) 
	&\leq \mathbb{P}( \theta_t(i)- \hat{\mu}_t(\gamma,i) >\frac{\Delta}{3}-2\sigma D(\gamma) | \mathcal{F}_{t-1}) ) 
	\\
	&\leq \frac{1}{2}\exp (-\frac{(\Delta/3-2\sigma D(\gamma))^2A(\gamma)}{2} )
	\\
	&\leq \frac{1}{2}(1-\gamma)^{144}.
	\end{aligned}
	\]
	For other $\mathcal{F}_{t-1}$, the  indicator term $\mathbbm{1}\{ N_t(\gamma,i)>A(\gamma), \hat{\mu}_t(\gamma,i)<x_t(i)+\sigma D(\gamma)\}$ will be $0$. Hence, we can bound the second part by $\frac{T}{2}(1-\gamma)^{144}.$

	\noindent {\bfseries {Step 4}} 
	Using Lemma \ref{new1}, we can get 
	\[
	\begin{aligned}
	&\sum_{t =1}^{T} \mathbb{P}(E_t(\gamma,i), \hat{\mu}_t(\gamma,i)<x_t(i)+\sigma D(\gamma), \theta_t(i)<y_t(i)-\sigma D(\gamma) )\\
	&\leq CT(1-\gamma)L(\gamma)\gamma^{-1/(1-\gamma)}\log\frac{1}{1-\gamma}.
	\end{aligned}
	\]
	Substituting  all into Equation \eqref {tmp3}, we can obtain the statement of Theorem \ref{result2}.
	

	\section{Experiments}
	
	In this section, we empirically compare the performance of DS-TS  w.r.t. the state-of-the-art  algorithms  on  Bernoulli and an arbitrarily generated bounded reward distributions. Specifically, we compare DS-TS with Thompson Sampling (TS)\cite{thompson1933likelihood} to evaluate the improvement obtained thanks to the employment of the discounted factor $\gamma$. We also compare DS-TS with  SW-TS \cite{trovo2020sliding} to evaluate the performance of sliding window and discounted factor.  Furthermore, we compare DS-TS with 
	another discounted method, DS-UCB \cite{garivier2011upper}, to evaluate the effect of Thompson Sampling and UCB. Moreover, we compare DS-TS with some novel and efficient algorithms such as CUSUM \cite{liu2018change}, M-UCB \cite{cao2019nearly} and LB-SDA \cite{baudry2021limited}. We measure the performance of each algorithm with the cumulative expected regret defined in Equation \eqref {regret}. The expected regret is averaged on $100$ independently runs.
	The 95\% confidence interval is obtained by  performing  $100$  independent runs and is shown as a semi-transparent region in the figure. 
	
	\subsection{Abruptly Changing Settings}
	\label{acs}
	
	{\bfseries{Experimental Setting  }}
	\quad The time horizon is set as $T=100000$. We split the time horizon into $5,10,20$ phases of equal length and use a number of arms $K = \{ 5,10,20,30\}$,respectively. We only show the results of some of the experimental settings, one can run the code on  online website to get more results\footnote{Our code is available at https://github.com/qh1874/nonmab }.
	
	\begin{figure}[!htbp] 
		\centering{\includegraphics[height=6cm,width=9cm]{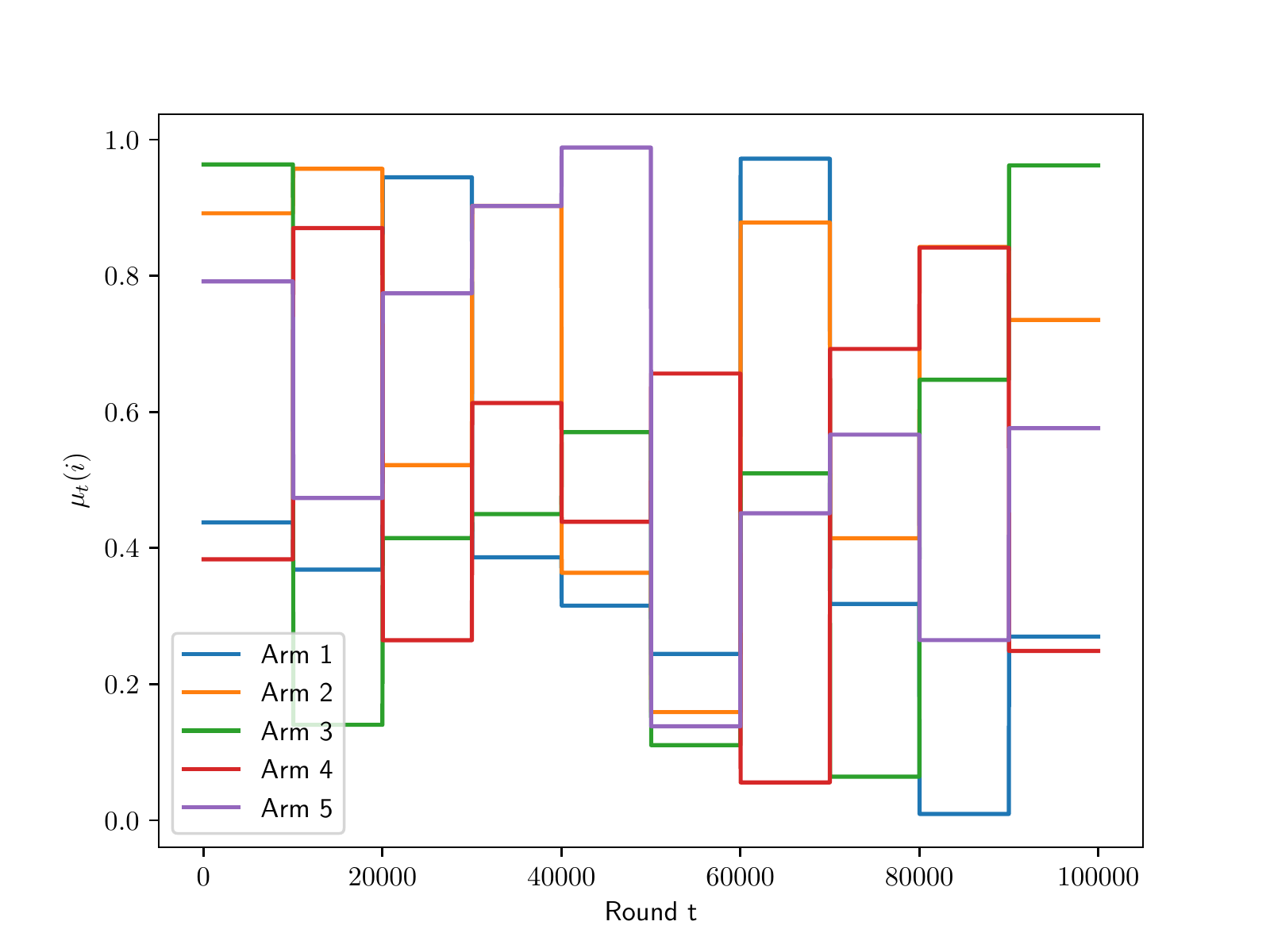}}
		\caption{$ K=5,B_T=10 $ for Bernoulli rewards.} 
		\label{fig_Ber}
	\end{figure}
	At each breakpoint, the expected value $\mu_t(i)$ of each arm $i$ is drawn from a uniform
	distribution over $[0, 1]$. In the stationary phase, the rewards distributions remain unchanged.
	The Bernoulli arms  for each phase are generated  as 
	$ \mu_t(i) \sim U(0,1). $
	Figure \ref{fig_Ber} depicts the expected rewards for Bernoulli arms with $K = 5$ and $B_T = 10$.

	Based on Corollary \ref{corollary1}, we set $\gamma= 1-\sqrt{B_T/T}$. $\tau_{max}$ is an important parameter that not only ensures the exploration ability of the algorithm but also prevents the sampling deviating too much from the arm's expectation.
	$\tau_{max}$ is generally $1/5-1/3$  of the upper bound of the expected rewards. In this experiment, the upper bound of expected rewards $\mu_{max}=1$, we take $\tau_{max}$ as $1/5$. To allow for fair comparison, DS-UCB uses the discount factor $\gamma= 1- \sqrt{B_T/T}/4, B=1,\xi=2/3$ suggested by \citeA{garivier2011upper}.  Based on \cite{baudry2021limited}, we set $\tau=2\sqrt{T\log(T)/B_T}$ for LB-SDA and SW-TS. For changepoint detection algorithm M-UCB, we set $w=800, b= \sqrt{w/2\log(2KT^2)}$ suggested by \citeA{cao2019nearly}. But set the amount of exploration $\gamma = \sqrt{KB_T\log(T)/T}$.  In practice, it has been found that using this value instead of the one guaranteed in \cite{cao2019nearly} will improve empirical performance \cite{baudry2021limited}. For CUSUM, following from \cite{liu2018change}, we set $\alpha= \sqrt{B_T/T\log(T/B_T)}$ and $h=\log(T/B_T)$. For our experiment settings, we choose $M=50, \epsilon=0.05$. Based on \cite{auer2002nonstochastic}, the parameters $\alpha$ and $\gamma$  for EXP3S are set as follows: $\alpha=1/T, \gamma=\min(1,\sqrt{K(e+B_T\log(KT)/((e-1)T))})$.
	
	\begin{figure}[!htbp] 
		\centering 
		\subfigure[]{ 
			\begin{minipage}[b]{0.45 \textwidth} 
				\centerline{	\includegraphics[width=7cm]{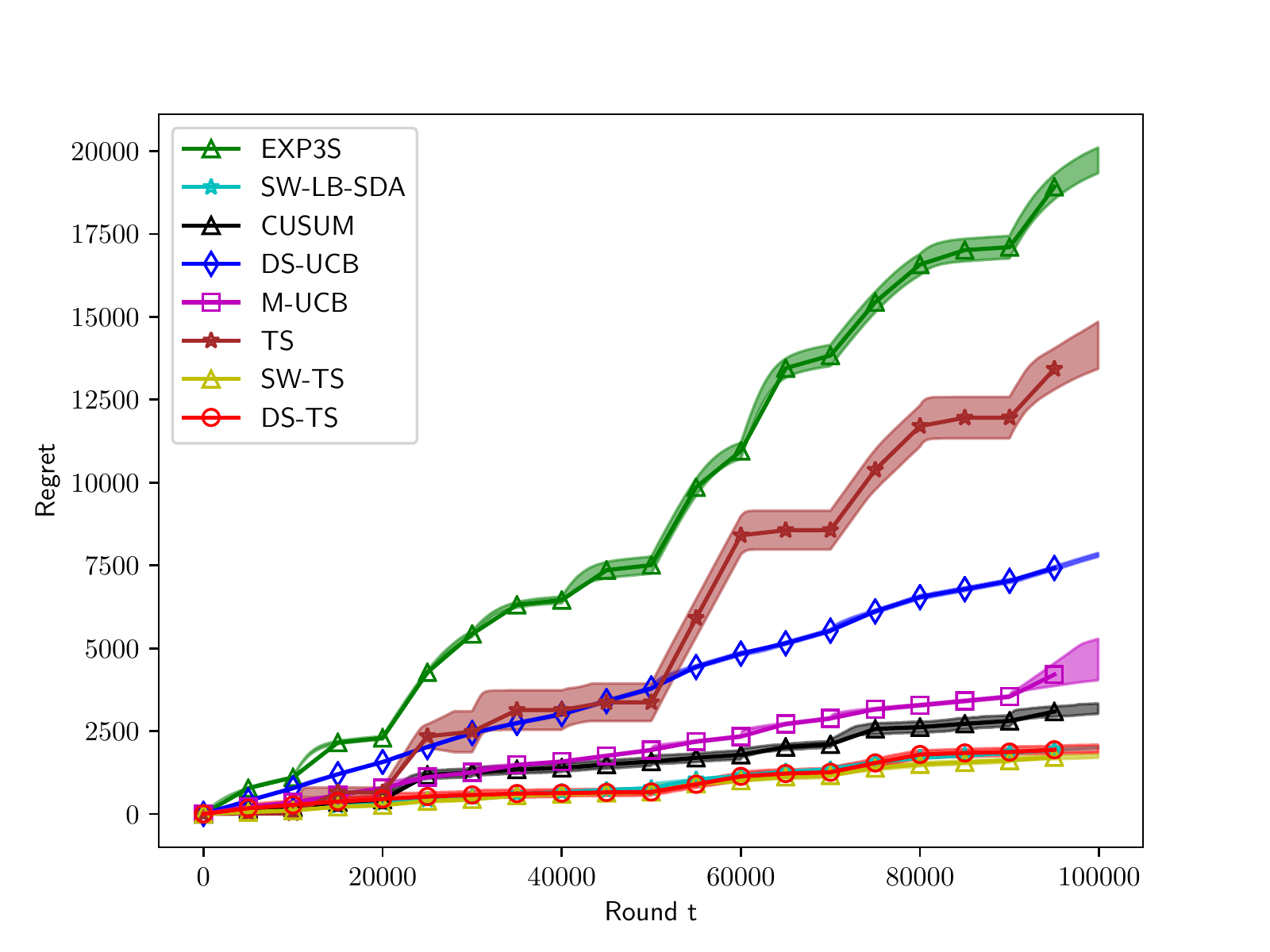}}
			\end{minipage} 
		} 
		\subfigure[]{
			\begin{minipage}[b]{0.45\textwidth}
				
				\centerline{	\includegraphics[width=7cm]{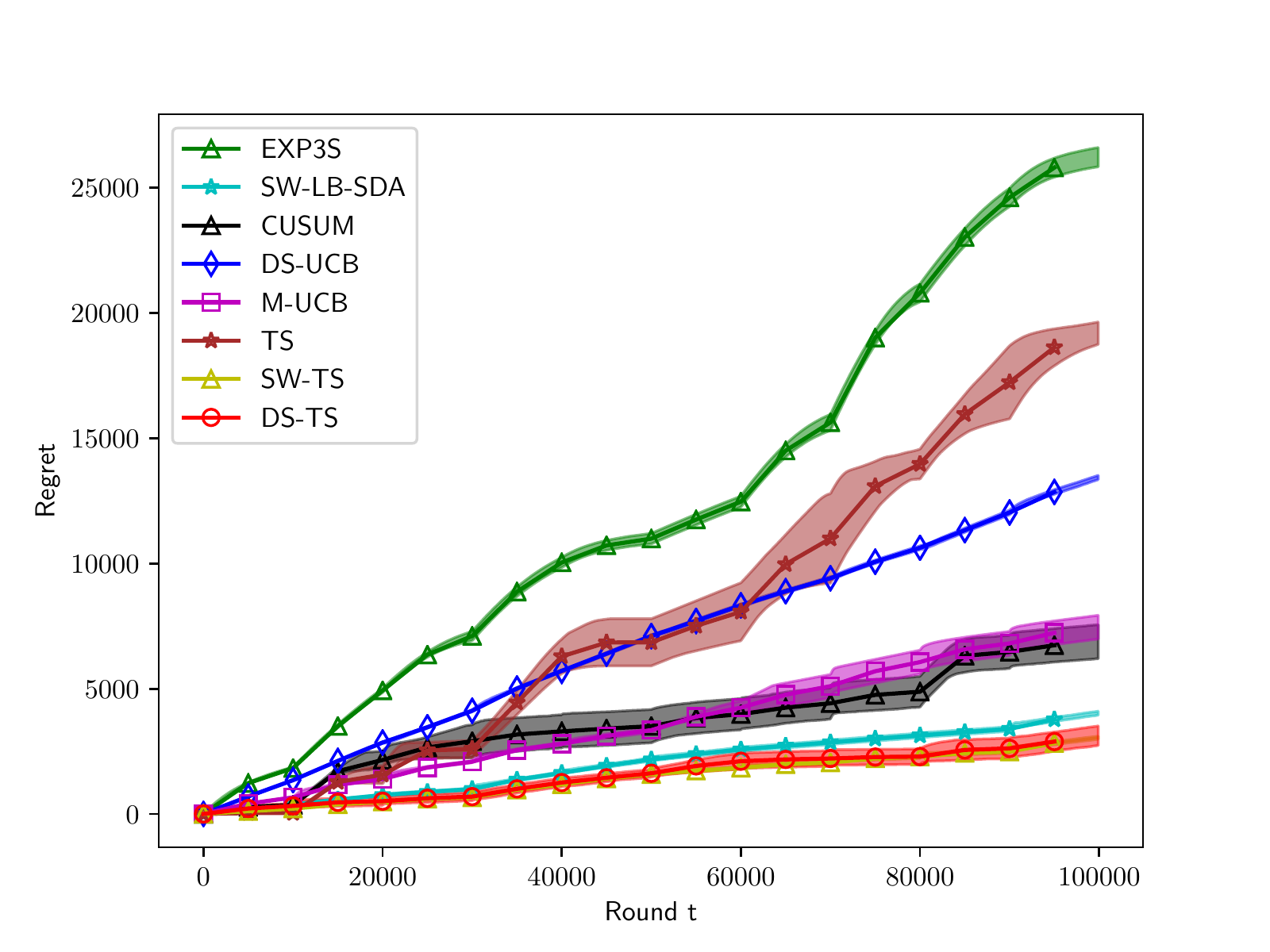} }
			\end{minipage} 	
		}
		\subfigure[]{ 
			\begin{minipage}[b]{0.45 \textwidth} 
				\centerline{	\includegraphics[width=7cm]{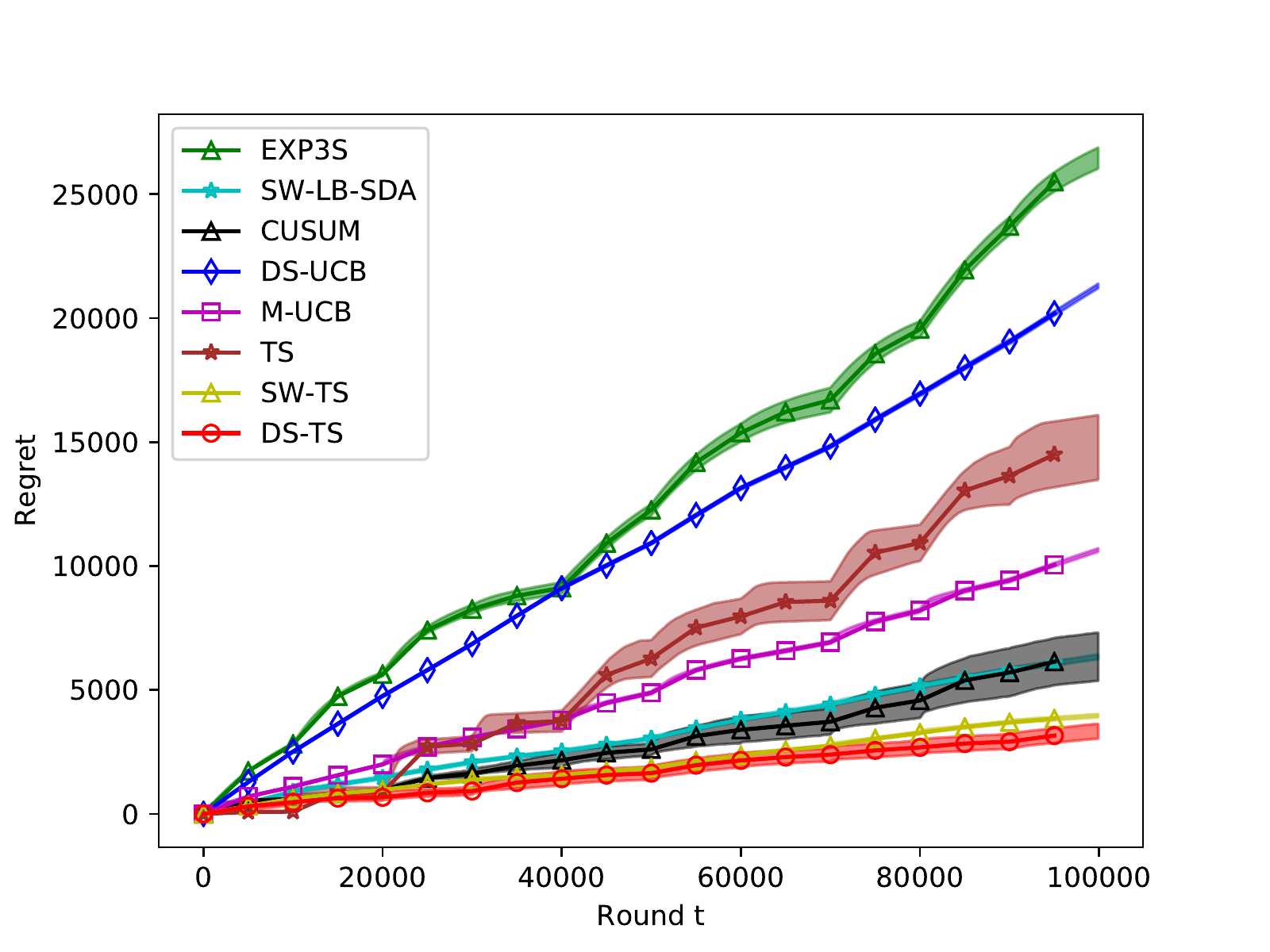}}
			\end{minipage} 
		} 
		\subfigure[]{
			\begin{minipage}[b]{0.45\textwidth}
				
				\centerline{	\includegraphics[width=7cm]{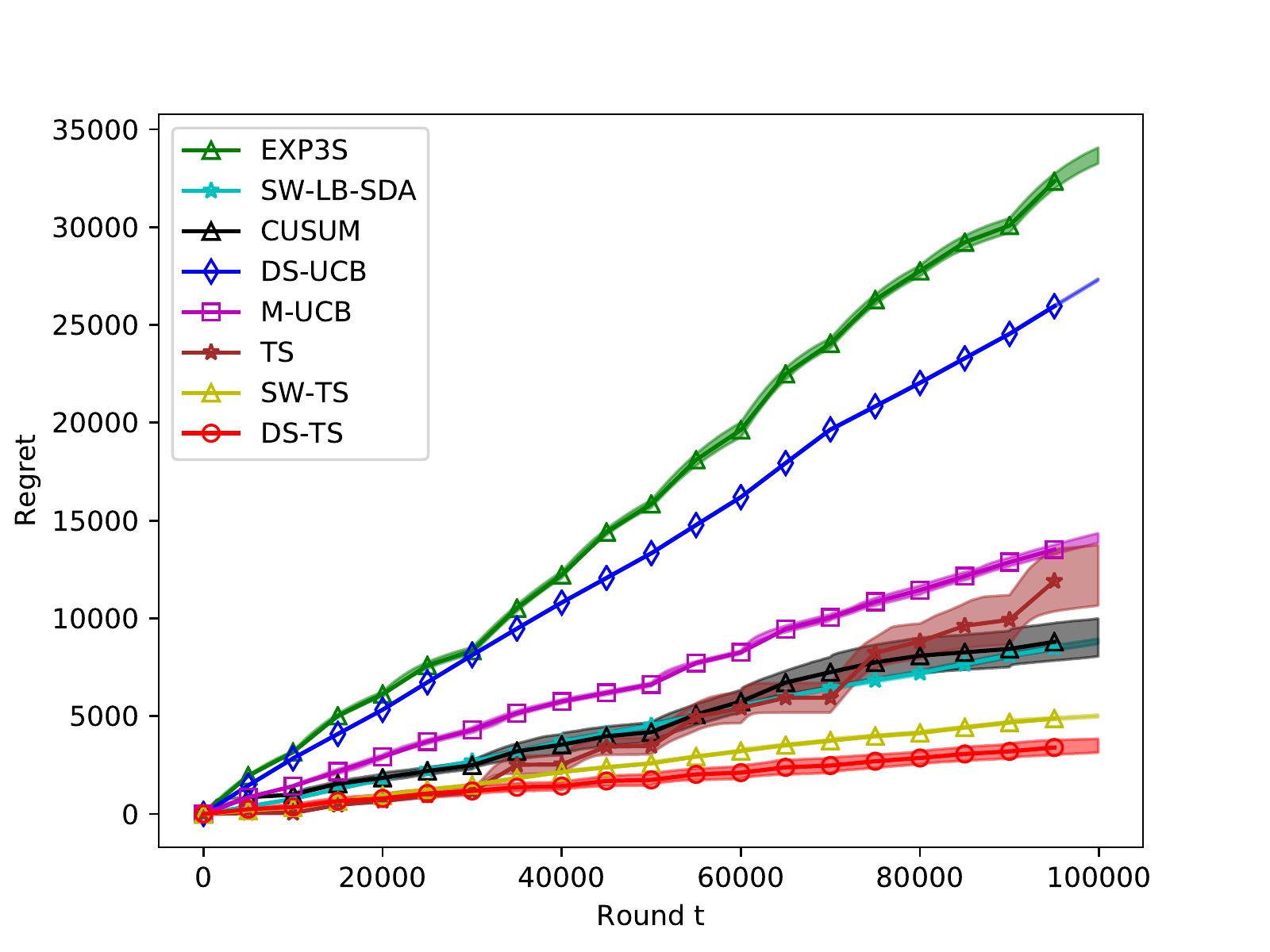} }
			\end{minipage} 	
		}
		
		\caption{Abruptly changing settings. Settings with $K=5,B_T=10$ (a), $K=10,B_T=10$ (b), $K=20,B_T=10 $ (c) and $K=30,B_T=10$ (d).  } 
		\label{fig:Bernoulli}
	\end{figure}
	
	\noindent{\bfseries{Results}}
	\quad Figure \ref{fig:Bernoulli}  report the results for Bernoulli arms in abruptly changing settings. It can be observed that our method and  SW-TS has almost the same performance.  Thompson Sampling (TS) is an algorithm for stationary MAB problems, so it oscillates a lot at the breakpoint. The changepoint detection algorithm CUSUM also shows competitive performance. Note that, our experiment does not satisfy the detectability assumption  of CUSUM \cite{liu2018change}. When the number of arms  are large, several algorithms, such as EXP3S and DS-UCB, have a near-linear regret, while our algorithm still performs well. While two TS-based algorithms, DS-TS and SW-TS, still work well. This is consistent with the results in \cite{bayati2020unreasonable}: when the number of arms is relatively large, algorithms based on TS are often better than that of UCB-class algorithms.
	
		\begin{figure*}[!htbp] 
		\centering 
		\subfigure[]{ 
			\begin{minipage}[b]{0.45 \textwidth} 
				\centerline{	\includegraphics[width=7cm]{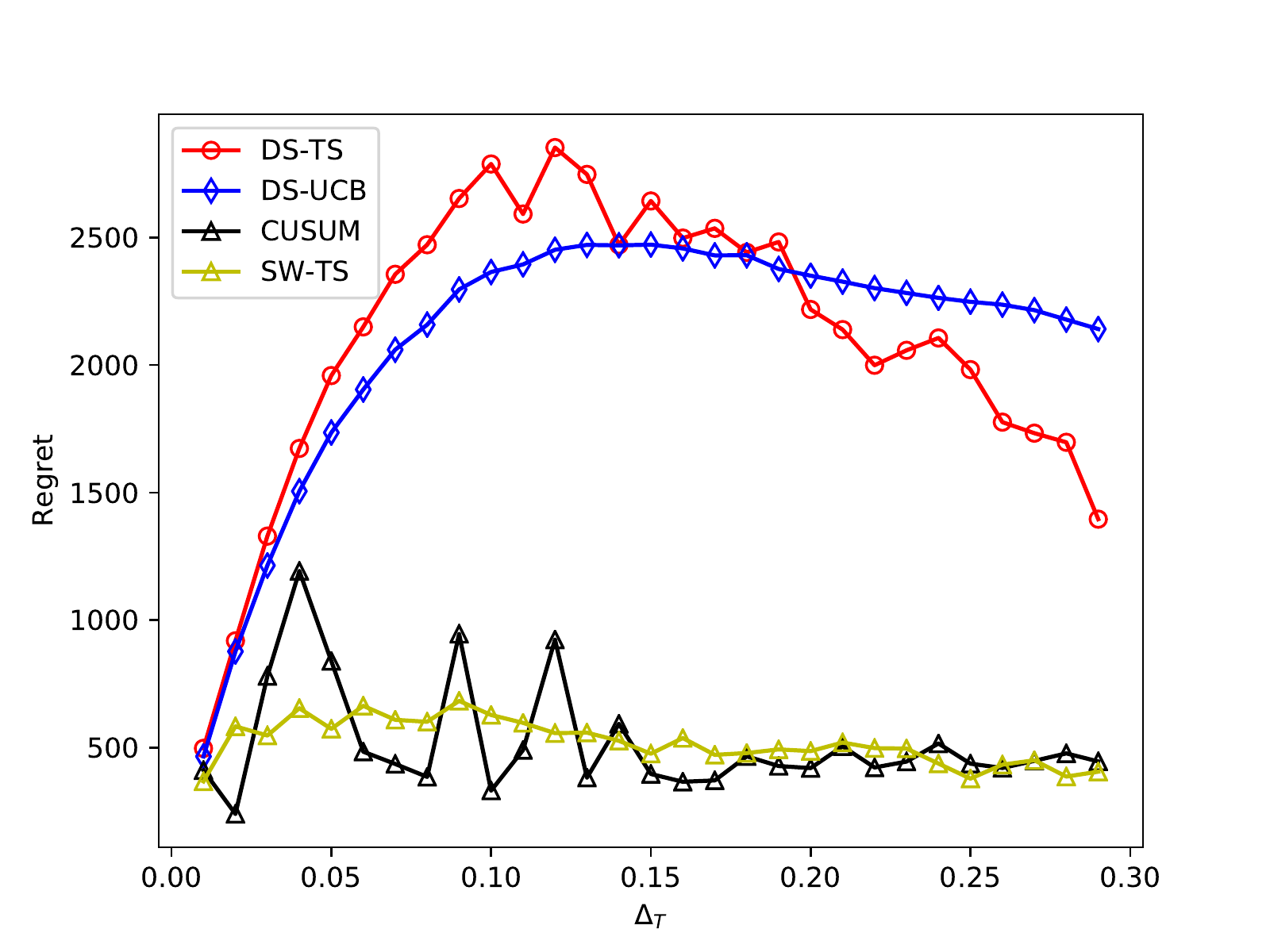}}
			\end{minipage} 
		} 
		\subfigure[]{
			\begin{minipage}[b]{0.45\textwidth}
				
				\centerline{	\includegraphics[width=7cm]{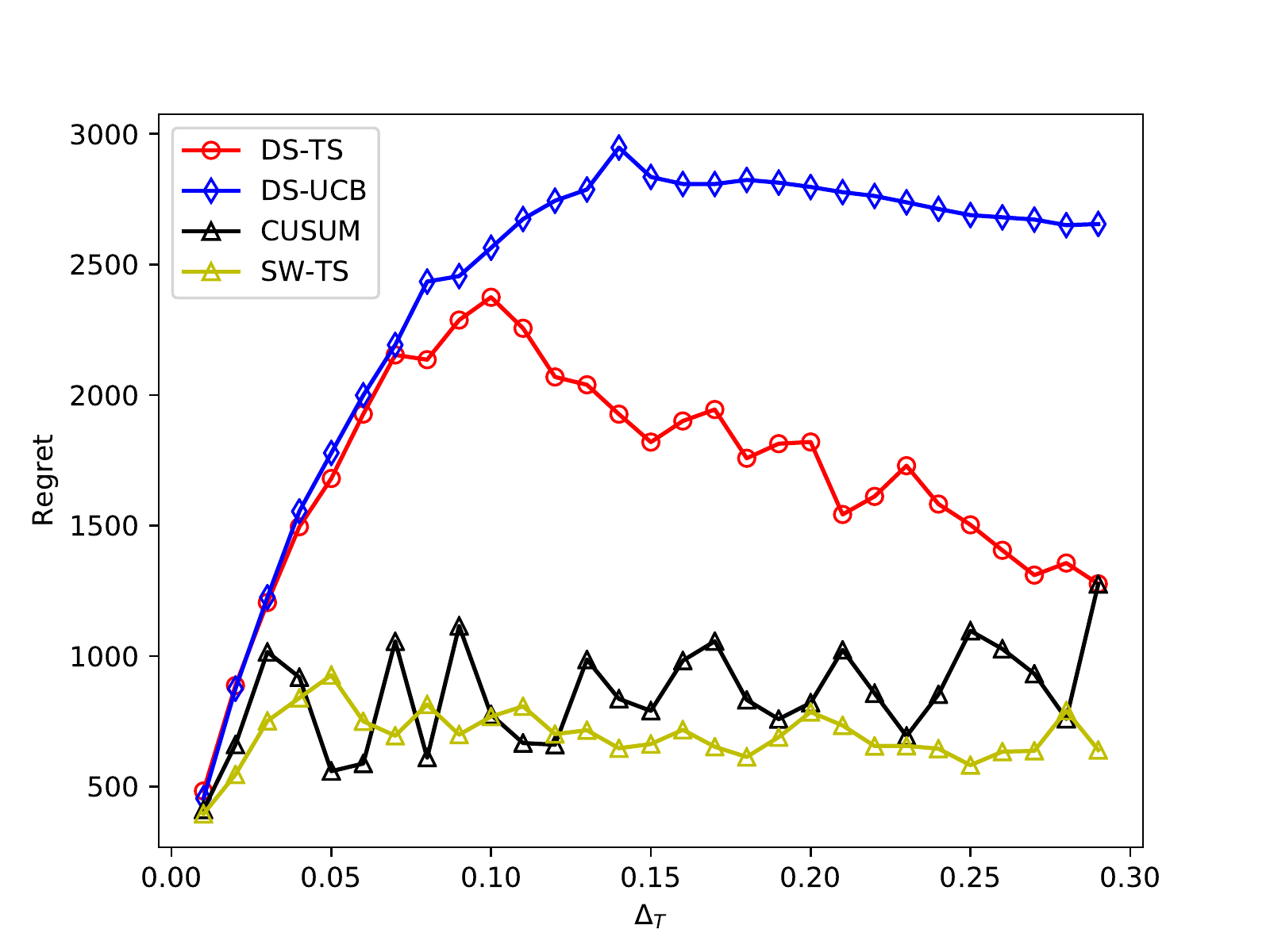} }
			\end{minipage} 	
		}
		\subfigure[]{ 
			\begin{minipage}[b]{0.45 \textwidth} 
				\centerline{	\includegraphics[width=7cm]{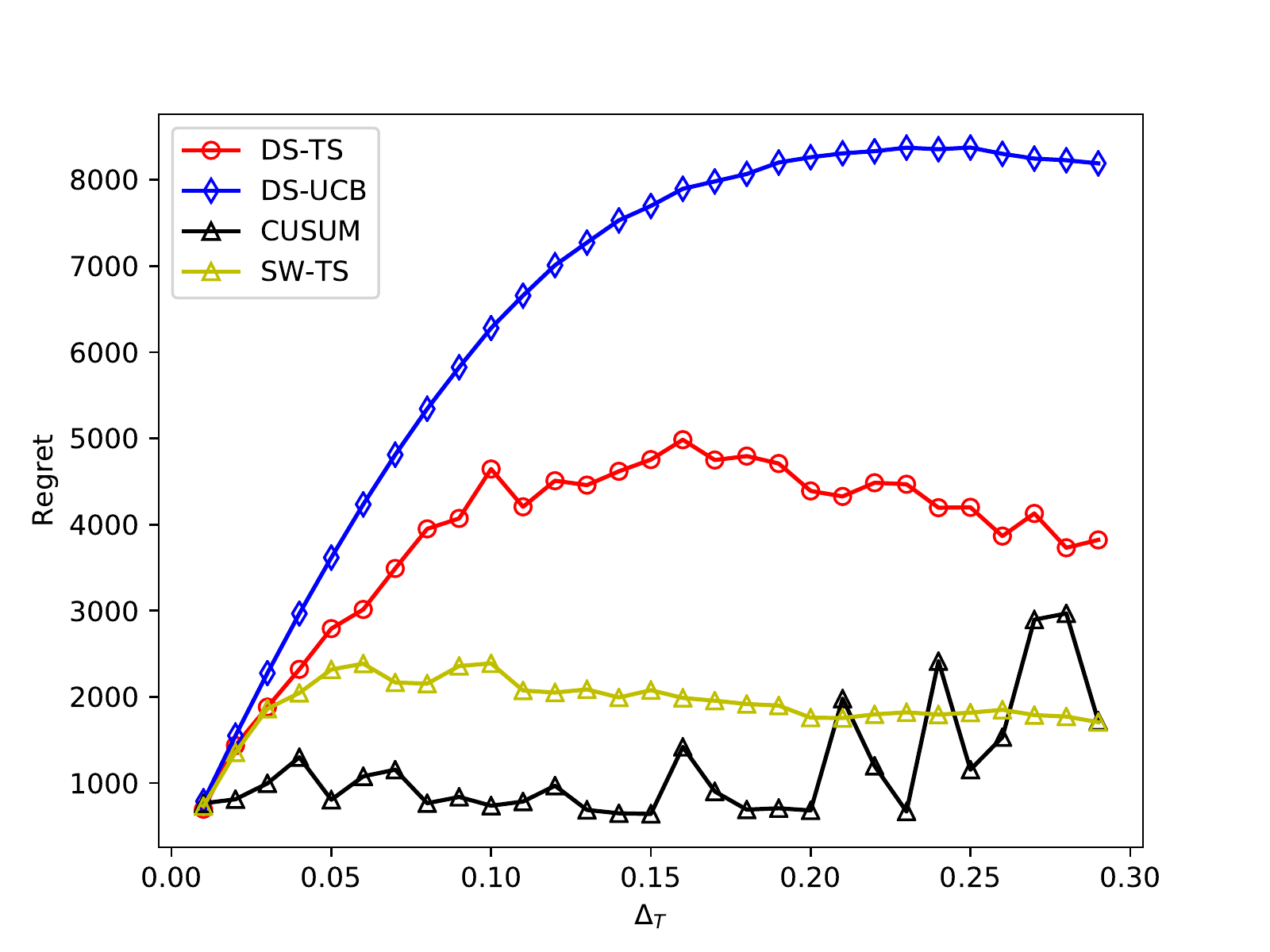}}
			\end{minipage} 
		} 
		\subfigure[]{
			\begin{minipage}[b]{0.45\textwidth}
				\centerline{	\includegraphics[width=7cm]{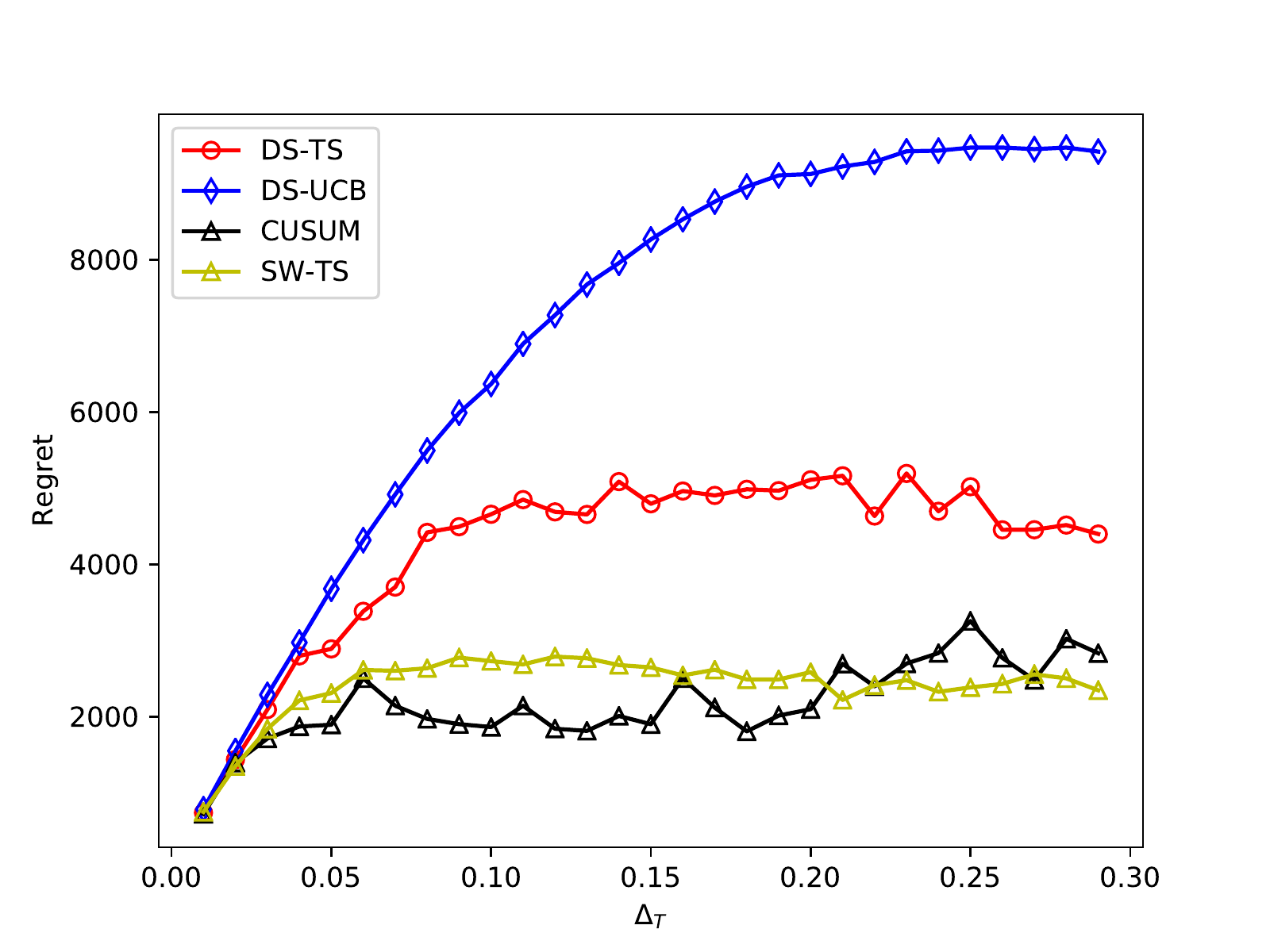} }
			\end{minipage} 	
		}
		
		\caption{$\Delta_T$ along with regret. Settings with $K=2,B_T=5$ (a), $K=2,B_T=10$ (b), $K=5,B_T=5 $ (c) and $K=5,B_T=10$ (d).  } 
		\label{fig:delta}
	\end{figure*}
	
	\noindent{\bfseries{Impact of $\boldsymbol{\Delta_T}$}}
	\quad The regret upper bound of DS-TS as well as CUSUM, DS-UCB and  SW-TS, all depend on $\Delta_T$.
	In order to more clearly analyze the impact of $\Delta_T$ on the performance, we consider  the environment: $K=\{2,5\},B_T=\{5,10\},T=100000$. We let $\Delta_T$ vary within the interval $(0,0.3)$, and compare the regrets of CUSUM, DS-UCB, SW-TS and DS-TS. 
	As shown in Figure \ref{fig:delta}, the performance of CUSUM and SW-TS is less significantly affected by $\Delta_T$.
	The reason behind this phenomenon is that their regret upper bound  have different factors about $\Delta_T$.  Their regret upper bounds are shown in Table \ref{table1}. As can be seen from the table, the upper bound of CUMSUM with respect to $\Delta_T$ is only related to $B_T$. The upper bound of SW-TS has factor $\frac{1}{\Delta_T}$, while DS-UCB and DS-TS have factor $\frac{1}{(\Delta_T)^2}$.
	However,
	DS-UCB and DS-TS also have smaller regret when $\Delta_T$ is small, as $\Delta_T$ increases, the regret value increases first and then tends to decrease or stabilize gradually. This is because when $\Delta_T$ is very small, although the algorithm cannot distinguish the optimal arm from the suboptimal arms well, the regret of choosing the suboptimal arm is small enough. 
	\begin{table}[!htp]
		\label{table1}
		\caption{ Comparison of regret bounds related to $\Delta_T$ in various algorithms }
		\centering
		\begin{tabular}{|c|c|c|c|c|}
			\hline
			Algorithm &CUMSUM& SW-TS & DS-UCB & DS-TS \\ \hline
			Bound &$\tilde{O}(\frac{B_T}{(\Delta_T)^2}+\sqrt{TB_T})$  & $\tilde{O}(\frac{\sqrt{TB_T}}{\Delta_T})$ & $\tilde{O}(\frac{\sqrt{TB_T}}{(\Delta_T)^2})$ & $\tilde{O}(\frac{\sqrt{TB_T}}{(\Delta_T)^2})$ \\ \hline
		\end{tabular}
	\end{table}

	\subsection{Smoothly Changing Settings}
	{\bfseries{Experimental Setting}} 
	We use a number of arms $K =\{5,10,20,30\}$ and the time horizon is set as $T=\{ 10^4,10^5\}$. The smoothly changing setting we use is the same as \cite{trovo2020sliding} and \cite{combes2014unimodal}, where the expected reward changing periodically according to the following function:
	\begin{equation}
		\label{tmp_eq}
	\begin{aligned}
		\mu_t(i) &=\frac{K-1}{K}-\frac{|w(t)-i|}{K}\\
		w(t)&=1+ \frac{(K-1)(1+sin(t\sigma))}{2}
	\end{aligned}
	\end{equation}
	The expected value generated from Equation \eqref {tmp_eq} clearly satisfies Assumption \ref{assumption_sigma}. \citeA{trovo2020sliding} have shown that $F=\frac{4K}{\sigma (K-1)}, \Delta_0=\frac{1}{3}$ satisfies Assumption \ref{assumption_delta} regardless of the value of $\beta$. 
	
	\begin{figure*}[!htbp] 
		\centering 
		\subfigure[]{ 
			\begin{minipage}[b]{0.45 \textwidth} 
				\centerline{	\includegraphics[width=7cm]{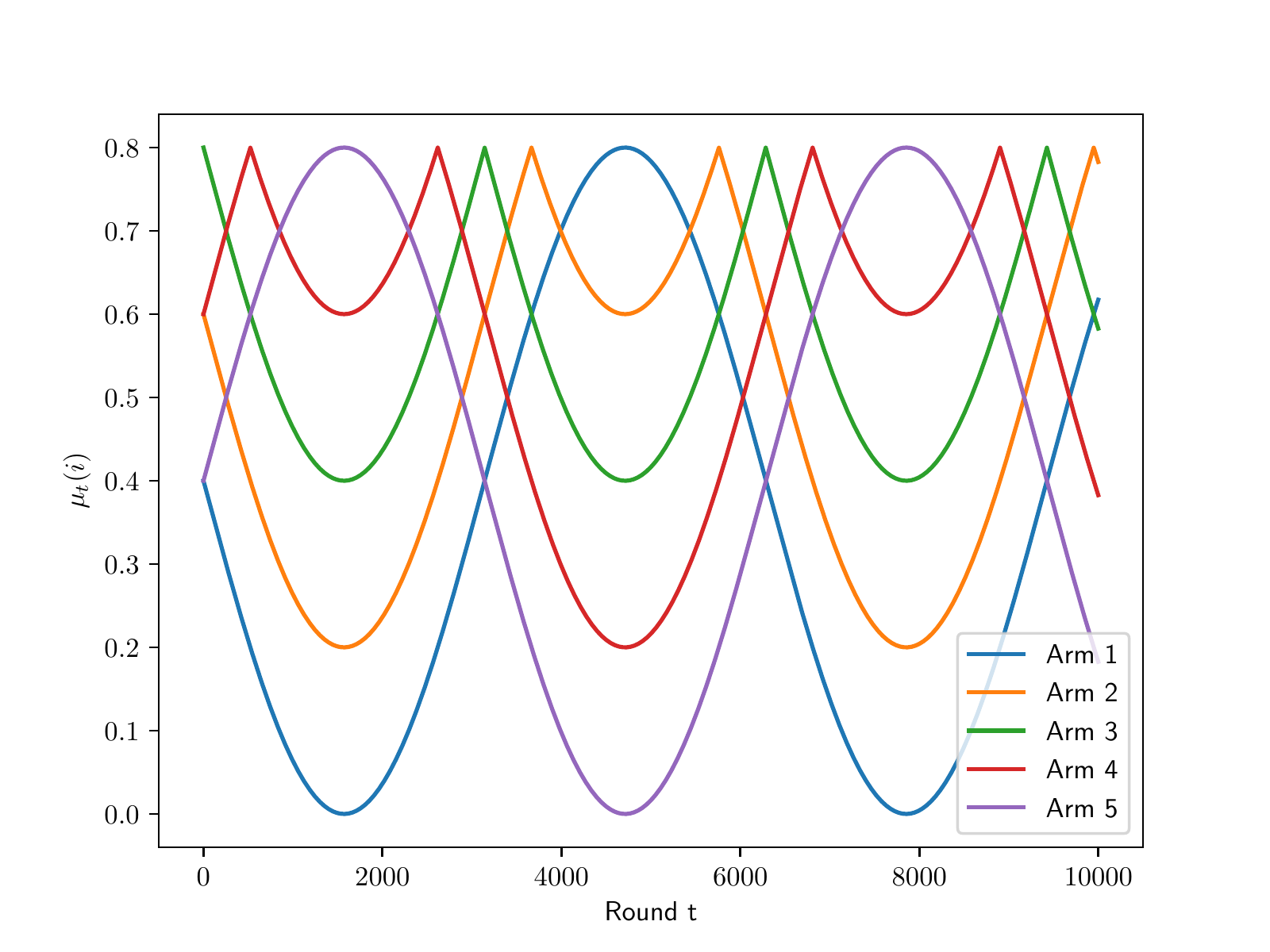}}
			\end{minipage} 
		} 
		\subfigure[]{
			\begin{minipage}[b]{0.45\textwidth}
				\centerline{	\includegraphics[width=7cm]{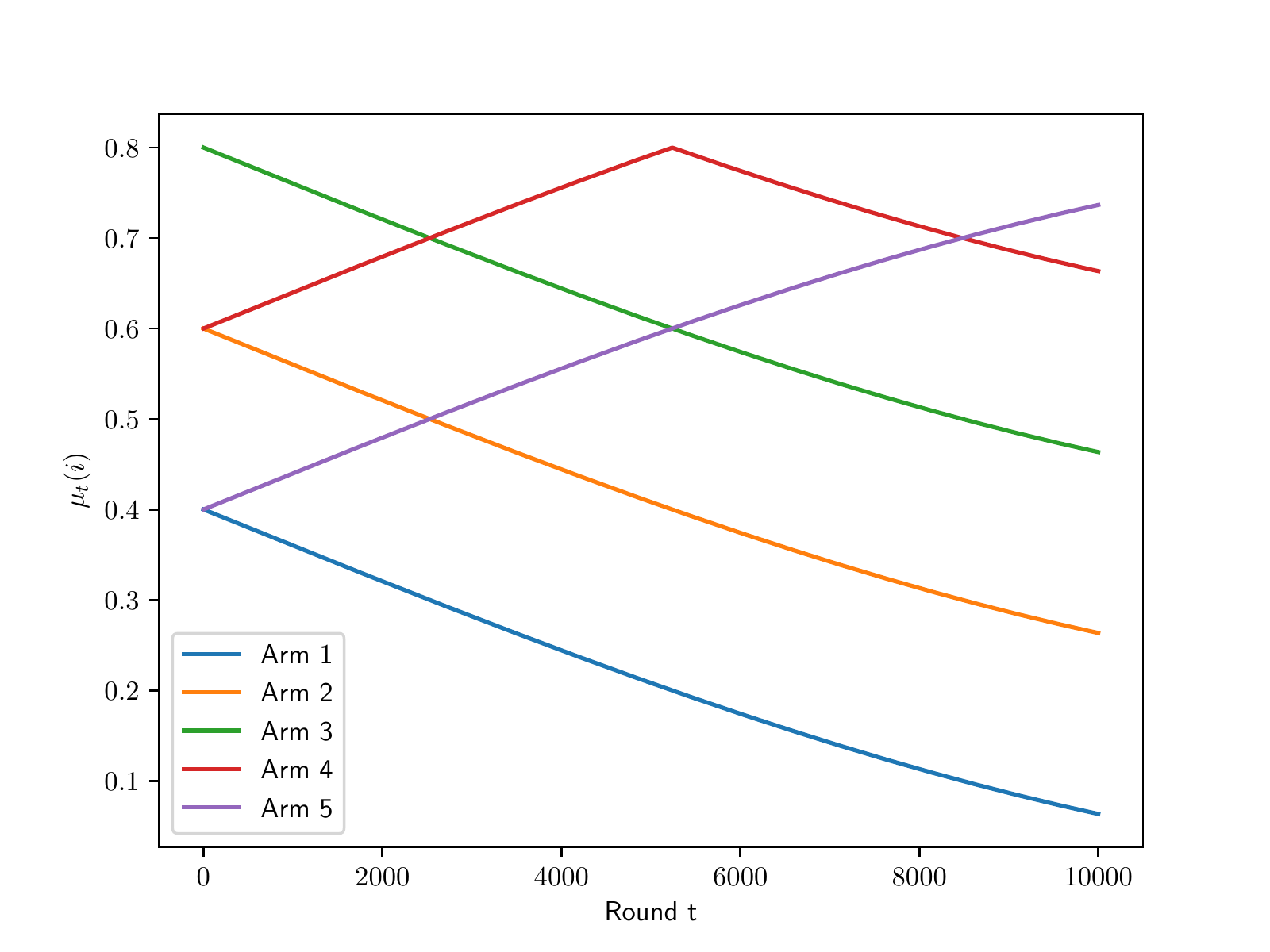} }
			\end{minipage} 	
		}
		
		\caption{Instances of expected rewards change over the time. Settings with $K=5,T=10^4,\sigma=0.001$ (a), $K=5,T=10^4,\sigma=0.0001$ (b).  } 
		\label{fig:Smoothly}
	\end{figure*}

	Unlike the abruptly changing settings, we do not compare the CUMSUM and M-UCB algorithms. Both algorithms use the opposite assumption to Assumption \ref{assumption_sigma} to guarantee detectability. 
	Let $B_T=1$(i.e. $b_1=1$, this means  there is no breakpoint), other algorithms including ours use the same parameter as abruptly changing settings. 
	Since we want to get a regret for $\tilde{O}(\sqrt{T})$, in this experiment we set $\beta=\frac{1}{2}$.  Corollary \ref{corollary2} suggests taking $\gamma=1-\frac{1}{\sqrt{T}}$, in this experiment we take $\gamma=1-\frac{10}{\sqrt{T}}$ to ensure that the conditions of Theorem \ref{result2} are  satisfied. 
	 In particular, if $T=10^4,\sigma=0.001$ or $T=10^5,\sigma=0.0001$, $\gamma=1-\frac{10}{\sqrt{T}}$ satisfies Theorem \ref{result2} while $\gamma=1-\frac{1}{\sqrt{T}}$ not. Figure \ref{fig:Smoothly} shows two instances of smoothly changing arms.
		\begin{figure*}[!htbp] 
		\centering 
		\subfigure[]{ 
			\begin{minipage}[b]{0.45 \textwidth} 
				\centerline{	\includegraphics[width=7cm]{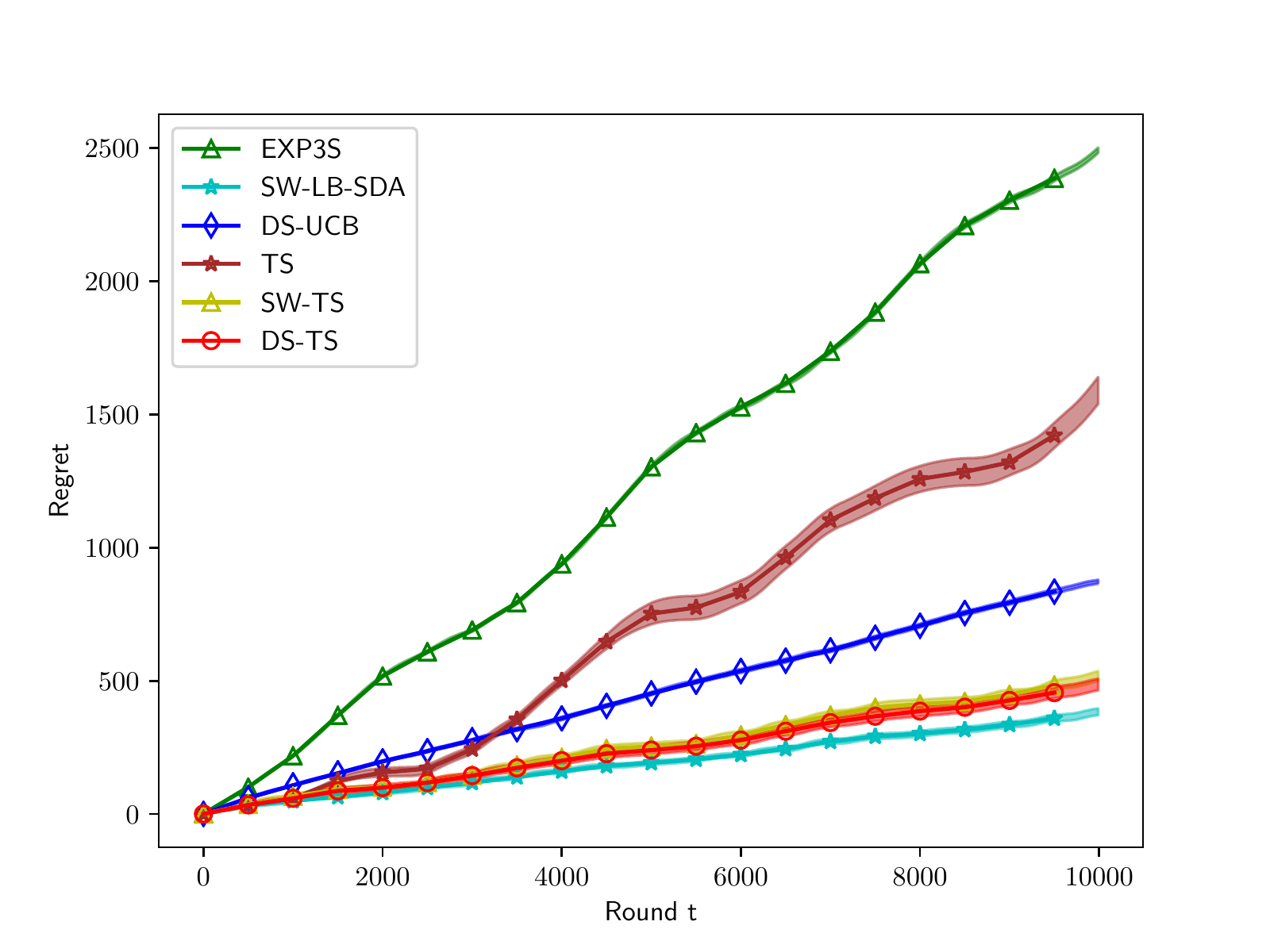}}
			\end{minipage} 
		} 
		\subfigure[]{
			\begin{minipage}[b]{0.45\textwidth}
				
				\centerline{	\includegraphics[width=7cm]{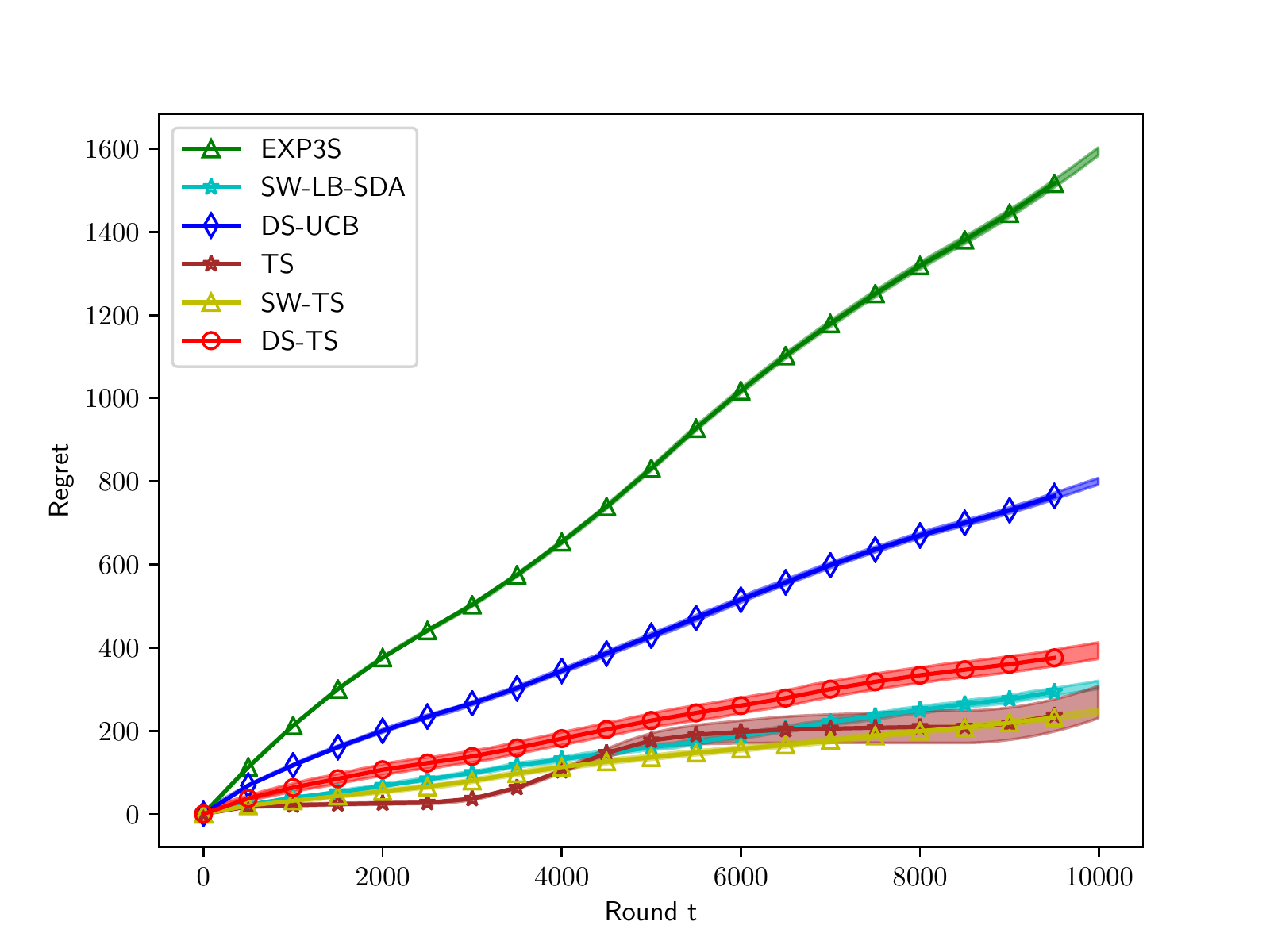} }
			\end{minipage} 	
		}
		\subfigure[]{ 
			\begin{minipage}[b]{0.45 \textwidth} 
				\centerline{	\includegraphics[width=7cm]{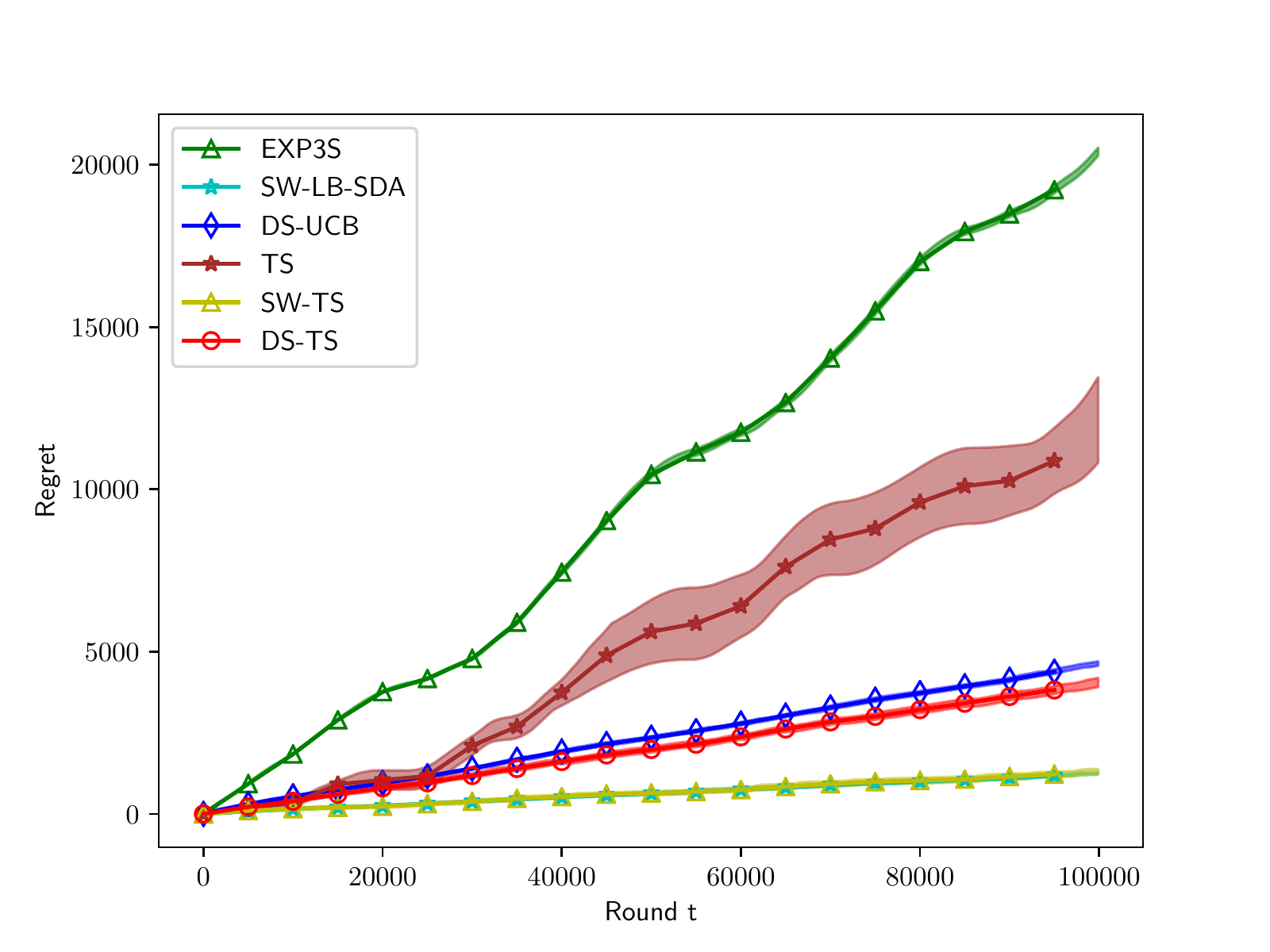}}
			\end{minipage} 
		} 
		\subfigure[]{
			\begin{minipage}[b]{0.45\textwidth}
				
				\centerline{	\includegraphics[width=7cm]{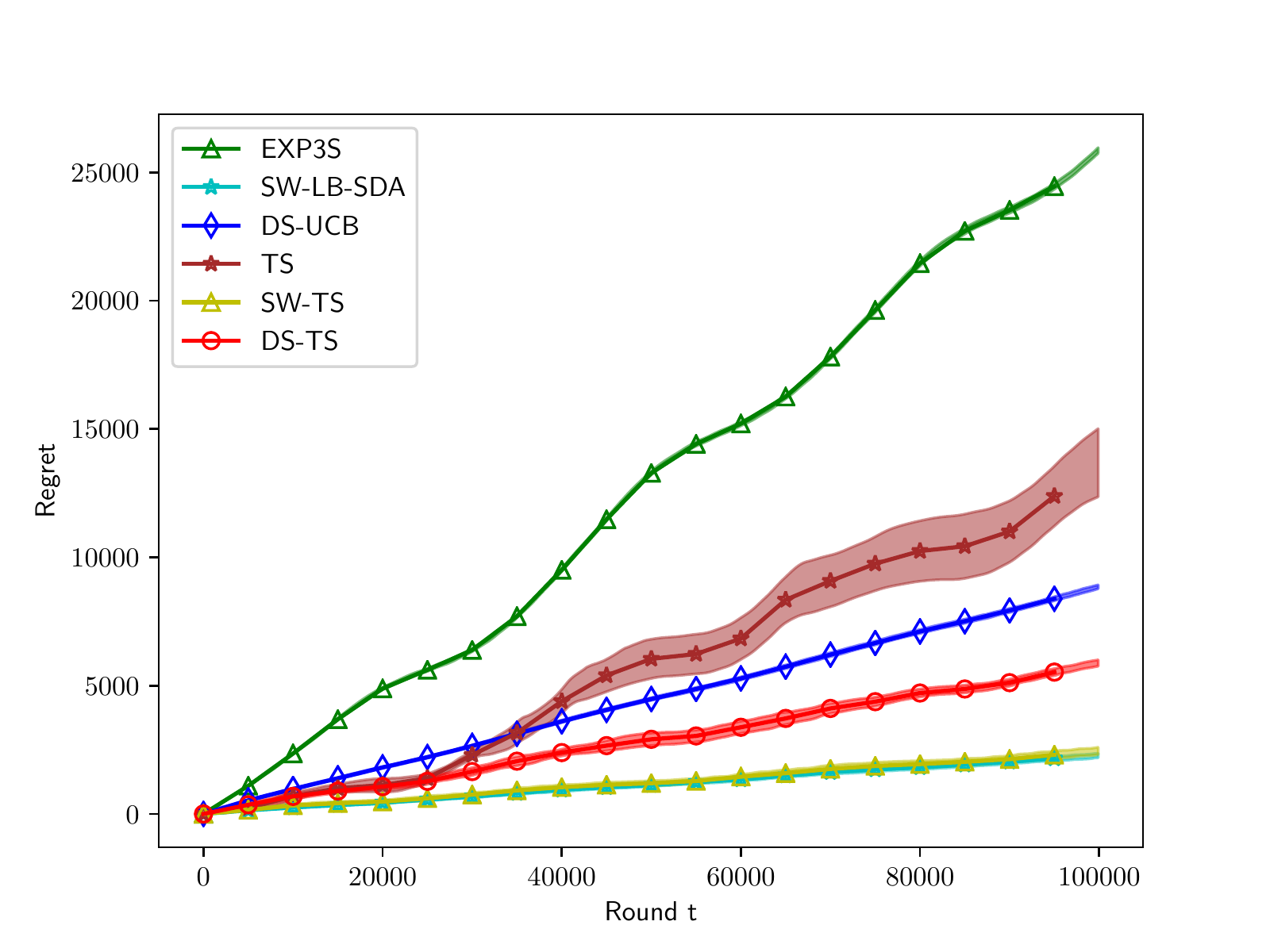} }
			\end{minipage} 	
		}
		
		\caption{Smoothly changing settings. Settings with $K=5,T=10^4,\sigma=0.001$ (a), $K=5,T=10^4,\sigma=0.0001$ (b), $K=5,T=10^5,\sigma=0.0001 $ (c) and $K=10,T=10^5,\sigma=0.0001$ (d).  } 
		\label{fig:Bernoulli_s}
	\end{figure*}
	
	\noindent{\bfseries{Results}} 
	
	Figure \ref{fig:Bernoulli_s}  report the results for smoothly changing settings. It can be seen that  SW-TS and  SW-LB-SDA achieve similar performance in several environmental settings. 
	Due to the extra logarithmic regret induced by the discounted method while adapting to changes in the reward, the performance of DS-TS is not as potent as that of SW-TS and SW-LB-SDA.
	However, when $T=10^4,\sigma=0.0001$, Thompson Sampling  exhibits excellent performance. The reason for this phenomenon can be explained by Figure \ref{fig:Smoothly}(b). In this environment the optimal arm is switched only twice and the difference between the optimal arm and the second best arm is not significant.

	\subsection{Prior Knowledge of $\boldsymbol{\mu_{max}}$}
	The expectation of  arm in our experiment is uniformly sampled from $(0,1)$.  For a reasonable number of arms, at least one of the arms has a expected value  close to 1. Most of the algorithms based on UCB and TS perform well due to their strong exploration ability. Now we change the experimental settings so that the expectation of the arms are relatively small and test the performance of each algorithm. 
	
	In abruptly changing settings, we limit the maximum expectation of the arms to less than $0.7$. 
	In smoothly changing settings, we limit the maximum expectation to less than $0.5$. To this end, we modify the expected arms generation function \eqref{tmp_eq} as $\mu_t(i)=\frac{K}{2(K-1)}\mu_t(i)$. We test the performance of each algorithm using the same parametrer  as before except DS-TS. We test DS-TS with $\tau_{max}=1/5$ and $\tau_{max}=\mu_{max}/5$  respectively. The latter means  DS-TS has additional information about $\mu_{max}$.
	
	\begin{figure*}[!htbp] 
		\centering 
		\subfigure[]{ 
			\begin{minipage}[b]{0.45 \textwidth} 
				\centerline{	\includegraphics[width=7cm]{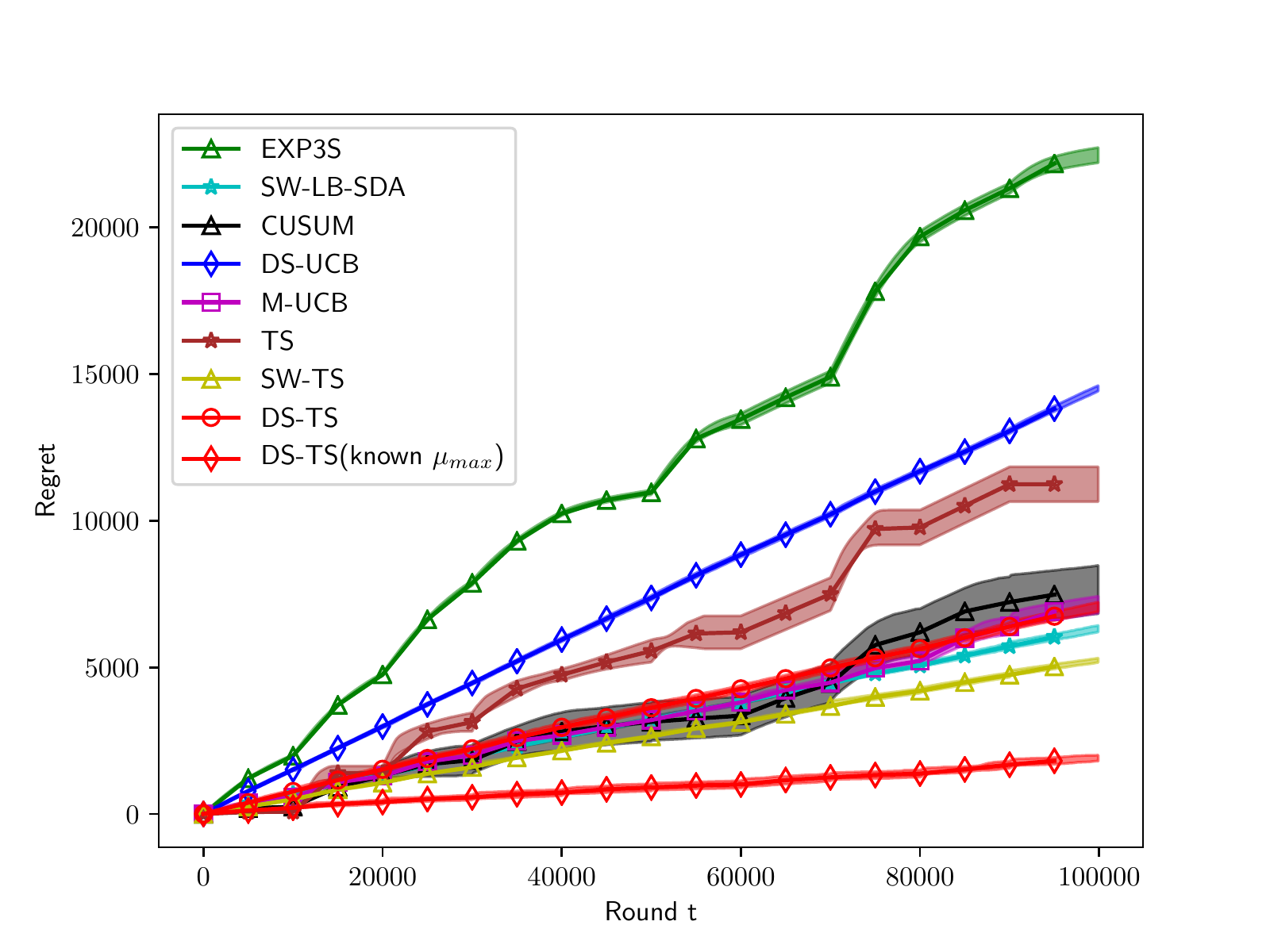}}
			\end{minipage} 
		} 
		\subfigure[]{
			\begin{minipage}[b]{0.45\textwidth}
				\centerline{	\includegraphics[width=7cm]{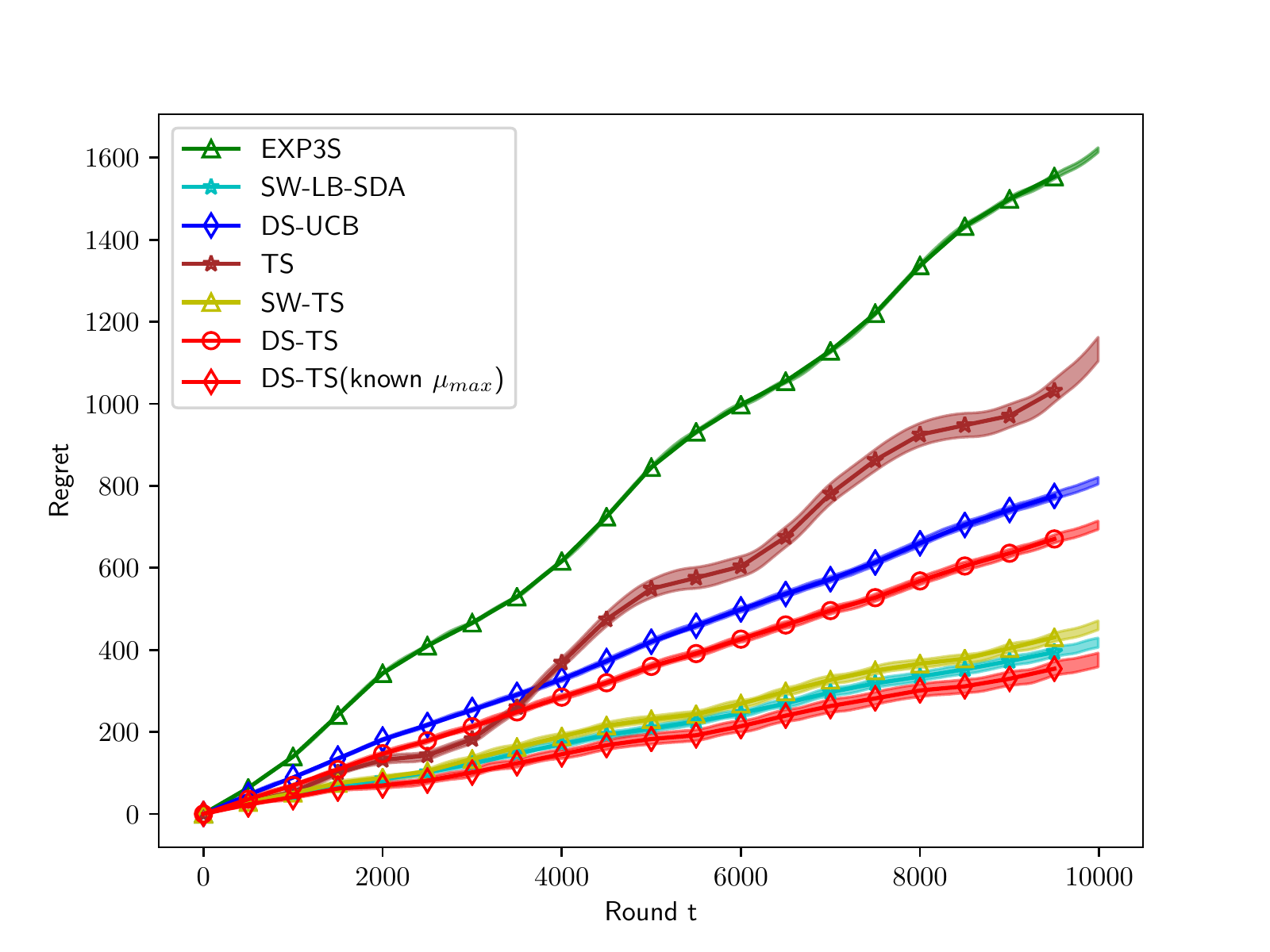} }
			\end{minipage} 	
		}
		
		\caption{ (a)  Abruptly changing settings with expected value less than 0.7. Settings with $K=10,B_T=10,T=10^5$.
			(b) Smoothly changing settings with expected value less than 0.5. Settings with $K=5,T=10^4,\sigma=0.001$. }
		\label{fig:ka}
	\end{figure*}
	
	Figure \ref{fig:ka} shows the  performance of each algorithm with small expected reward.  First, a comparison of Figure \ref{fig:Bernoulli}(b) shows that the regret values for each algorithm are larger than the case where the expected rewards are sampled uniformly from $(0,1)$. In smoothly changing settings, the expected reward was scaled down, meaning that the reward of the arm changed more slowly. Comparison with Figure 1 reveal that the regret of each algorithm even decreased except for DS-TS. This suggests that DS-TS is more influenced by $\mu_{max}$ than the other algorithms. Second, if the  $\mu_{max}$ is known in advance and the $\tau_{max}$  of DS-TS is reset to $\mu_{max}/5$, the performance of DS-TS will be significantly improved.

	\section{Discussion}

	In this paper, we have proposed DS-TS algorithm with Gaussian priors for abruptly changing and smoothly changing MAB problems. Under mild assumptions, 
	we  provide the regret upper bounds of DS-TS  in both non-stationary settings. Our experiments show that DS-TS can achieve significant regret reduction with respect to the state-of-the-art algorithms when the $\mu_{max}$ is known in advance. Furthermore, we empirically analyze the influence of $\Delta_T$ on the performance of different algorithms. 
	
	However, there are still some shortcomings in our work. First, the performance of DS-TS  is usually similar to that of SW-TS in abruptly changing settings, but the regret upper bound of DS-TS has an extra logarithmic term $\log T$, which is probably due to the fact that our analysis method yields a too rough upper bound. Second, it's natural to use the Bernoulli prior for bounded rewards. As $N_t(\gamma,i)$ is no longer a positive integer in  discounted method, the relationship between binomial distribution and beta distribution cannot be used for analysis. According to the literature on  tail bounds of the beta distribution \cite{zhang2020non}, the method in this paper cannot be used to analyze DS-TS with Bernoulli priors. Addressing these shortcomings could be a future research direction.


	\bibliographystyle{theapa}
	\bibliography{sample}

	\newpage
	\appendix
	
	\section{Facts and Lemmas}
	The following inequality is the  anti-concentration and  concentration bound for  Gaussian distributed random variables.
	\begin{fact}[\cite{abramowitz1964handbook}]
			For a Gaussian distributed random variable $X$ with mean $\mu$ and variance $\sigma^2$, for any $a>0$
		\[
		\frac{1}{\sqrt{2\pi}}\frac{a}{1+a^2}e^{-a^2/2} \leq \mathbb{P}(X-\mu>a\sigma)\leq 	\frac{1}{a+\sqrt{a^2+4}}e^{-a^2/2}	
		\]\label{fact1}
	\end{fact}
	
		The following lemma is adapted from \cite{agrawal2013further} and is often used in the analysis of Thompson Sampling, can transform the probability of selecting the $i$th arm into the probability of selecting the optimal arm $i_t^*$.
	\begin{lemma}
		\label{pit}
		Let $p_{i,t}=\mathbb{P}(\theta_t(*)>y_t(i)|\mathcal{F}_{t-1})$. For any $A>0$, $i \neq i_t^*$, 
		\[ 
		\mathbb{P}(i_t=i, \theta_t(i) < y_t(i) |\mathcal{F}_{t-1}) \leq \frac{(1-p_{i,t})}{p_{i,t}} \mathbb{P}(i_t=i_t^*, \theta_t(i) < y_t(i) |\mathcal{F}_{t-1})
		\]
	\end{lemma}
	
	\begin{lemma}[\cite{garivier2011upper}]
		\label{N<A}
		For any $i \in \{1,...,K\}$,  $\gamma \in (0,1)$ and $A>0$,
		\[ \sum_{t=1}^{T}\mathbbm{1}\{ i_t=i,N_t(\gamma,i)<A \} \leq \lceil T(1-\gamma)  \rceil A\gamma^{-1/(1-\gamma)}. \]
	\end{lemma}
	
	\begin{lemma}
		\label{hh}
		For abruptly changing settings,
		\[ \mathbb{P}(\hat{\mu}_t(\gamma,i)>\mu_t(i)+\frac{\Delta_t(i)}{3},N_t(\gamma,i)>A(\gamma))
		\leq (1-\gamma)^{48}\log\frac{1}{1-\gamma}
		\]
		For smoothly changing settings,
		\[ \mathbb{P}(\hat{\mu}_t(\gamma,i)>\mu_t(i)+\frac{\Delta_t(i)}{3}+\sigma D(\gamma),N_t(\gamma,i)>A(\gamma))
		\leq (1-\gamma)^{48}\log\frac{1}{1-\gamma}
		\]

	\end{lemma}
	\begin{proof}			
		Recall that, $m=\frac{12\sqrt{2}}{\sqrt{1-\gamma}}+3,n=12\sqrt{2}+3\sqrt{1-\gamma}$.
		In the abruptly changing settings, $A(\gamma)=\frac{n^2\log(\frac{1}{1-\gamma})}{(\Delta_T)^2}$. We  have
		\begin{equation}
			\begin{aligned}
				&\mathbb{P}(\hat{\mu}_t(\gamma,i)>\mu_t(i)+\frac{\Delta_t(i)}{3},N_t(\gamma,i)>A(\gamma))
			\\
				&=\mathbb{P}(\hat{\mu}_t(\gamma,i)-\ddot{\mu}_t(\gamma,i)>\mu_t(i)+\frac{\Delta_t(i)}{3}-\ddot{\mu}_t(\gamma,i),N_t(\gamma,i)>A(\gamma))\\
				& \stackrel{(a)}{\leq}
				\mathbb{P}(
				\frac{N_t(\gamma,i)(\hat{\mu}_t(\gamma,i)-\ddot{\mu}_t(\gamma,i))}{\sqrt{N_t(\gamma^2,i)}}>\frac{N_t(\gamma,i)}{\sqrt{N_t(\gamma^2,i)}}(\frac{\Delta_T}{3}-U_t(\gamma,i)),N_t(\gamma,i)>A(\gamma))\\	&\stackrel{(b)}{\leq} 
				\mathbb{P}(
				\frac{N_t(\gamma,i)(\hat{\mu}_t(\gamma,i)-\ddot{\mu}_t(\gamma,i))}{\sqrt{N_t(\gamma^2,i)}}> (\frac{1}{3}-\frac{1}{m})\Delta_T\sqrt{A(\gamma)}   )\\
				&\stackrel{(c)}{\leq} \frac{\log \frac{1}{1-\gamma}}{\log(1+\eta)} \exp(-2 ((\frac{1}{3}-\frac{1}{m})\Delta_T)^2 A(\gamma) (1-\frac{\eta^2}{16}))\\
				&\leq \frac{\log \frac{1}{1-\gamma}}{\log(1+\eta)} \exp(-64\log(\frac{1}{1-\gamma}) (1-\frac{\eta^2}{16}))
			\end{aligned}
		\end{equation}
		where (a) uses Lemma \ref{D}, (b) follows from $N_t(\gamma,i)>N_t(\gamma^2,i),\Delta_T<\Delta_t(i)$ and Equation \eqref {UA}, (c) uses the self-normalized  Hoeffding-type inequality\cite{garivier2011upper}.
		Let $\eta=2$, we can obtain the statement for abruptly changing settings.
		
		For smoothly changing settings,
		\begin{equation}
			\begin{aligned}
				&\mathbb{P}(\hat{\mu}_t(\gamma,i)>\mu_t(i)+\frac{\Delta_t(i)}{3}+\sigma D(\gamma),N_t(\gamma,i)>A(\gamma))
				\\
				&=\mathbb{P}(\hat{\mu}_t(\gamma,i)-\ddot{\mu}_t(\gamma,i)>\mu_t(i)+\frac{\Delta_t(i)}{3}+\sigma D(\gamma)-\ddot{\mu}_t(\gamma,i),N_t(\gamma,i)>A(\gamma))\\
				& \stackrel{(d)}{\leq}
				\mathbb{P}(
				\frac{N_t(\gamma,i)(\hat{\mu}_t(\gamma,i)-\ddot{\mu}_t(\gamma,i))}{\sqrt{N_t(\gamma^2,i)}}>\frac{N_t(\gamma,i)}{\sqrt{N_t(\gamma^2,i)}}(\frac{\Delta}{3}-U_t(\gamma,i)),N_t(\gamma,i)>A(\gamma))\\	
				&\leq \frac{\log \frac{1}{1-\gamma}}{\log(1+\eta)} \exp(-64\log(\frac{1}{1-\gamma}) (1-\frac{\eta^2}{16}))
			\end{aligned}
		\end{equation}
	where (d) uses the Lemma \ref{D1}. Let $\eta=2$, we obtain the conclusion.
	\end{proof}

	\section{Proofs of Lemmas}
	\subsection{Proof of Lemma \ref{D}}
	Let $M_t(\gamma,i)= \sum_{j=1}^{t} \gamma^{t-j} \mu_j(i) \mathbbm{1}\{ i_j=i \}$, 
	$\ddot{\mu}_t(\gamma,i)=\frac{M_t(\gamma,i)}{N_t(\gamma,i)} \in [0,1]$ is a convex combination of elements $\mu_j(i),j=1,...,t$. For $t \in \mathcal{T}(\gamma)$, 
	\[
	\begin{aligned}
		|\mu_t(i)-\ddot{\mu}_t(\gamma,i)|
		&=\frac{1}{N_t(\gamma,i)}| M_t(\gamma,i)-\mu_t(i)N_t(\gamma,i)| \\
		&= \frac{1}{N_t(\gamma,i)} | \sum_{j=1}^{t-D(\gamma)} \gamma^{t-j}(\mu_j(i)-\mu_t(i))\mathbbm{1}\{ i_j=i \} | \\
		&\leq \frac{1}{N_t(\gamma,i)}  \sum_{j=1}^{t-D(\gamma)}  \gamma^{t-j}\mathbbm{1}\{ i_j=i \}\\
		&=\frac{1}{N_t(\gamma,i)} \gamma^{D(\gamma)}N_{t-D(\gamma)}(\gamma,i)\\
		&\stackrel{(a)}{\leq} \frac{ \gamma^{D(\gamma)}}{N_t(\gamma,i)(1-\gamma)} \\
		&\stackrel{(b)}{\leq}  \sqrt{\frac{ \gamma^{D(\gamma)}}{N_t(\gamma,i)(1-\gamma)}},
	\end{aligned}
	\]
	where (a) follows from $N_{t-D(\gamma)}(\gamma,i) \leq \frac{1}{1-\gamma}$, (b) follows from $ |\mu_t(i)-\ddot{\mu}_t(\gamma,i)| \leq 1 $ and $ 1 \wedge x \leq \sqrt{x} $. By the definition of $D(\gamma)$, 
	\[|\mu_t(i)-\ddot{\mu}_t(\gamma,i)|\leq \sqrt{\frac{-(1-\gamma)\log (1-\gamma)}{N_t(\gamma,i)}} \]
	
	\subsection{Proof of Lemma \ref{new}}
	Recall that $p_{i,t}=\mathbb{P}(\theta_t(*)>y_t(i)|\mathcal{F}_{t-1}), A(\gamma)=\frac{n^2\log(\frac{1}{1-\gamma})}{(\Delta_T)^2},U_t(\gamma,i)=\sqrt{\frac{(1-\gamma)\log \frac{1}{1-\gamma}}{N_t(\gamma,i)}}, m=\frac{12\sqrt{2}}{\sqrt{1-\gamma}}+3,n=12\sqrt{2}+3\sqrt{1-\gamma}
	$, function $F(x)=\frac{1}{\sqrt{2\pi}}\frac{x}{1+x^2}e^{-x^2/2}.$
	
	Our algorithm uses Gaussian priors $\theta_t(i) \sim \mathcal{N}(\hat{\mu}_t(i),\min \{\frac{1}{N_t(\gamma,i)} ,\tau_{max}^2\})$.

	If $N_t(\gamma,i) < \frac{1}{\tau_{max}^2}$, then $\theta_t(i)$ is sampling from $\mathcal{N}(\hat{\mu}_t(i),\tau_{max}^2)$.
	We have 
	\[ p_{i,t}=\mathbb{P}( \theta_t(*)>y_t(i) | \mathcal{F}_{t-1}) \geq \mathbb{P}( \theta_t(*)-\hat{\mu}_t(*)>\mu_t(*) | \mathcal{F}_{t-1}) \geq \mathbb{P}( \theta_t(*)-\hat{\mu}_t(*)>\mu_{max} | \mathcal{F}_{t-1})	\]
	Using Fact \ref{fact1}, $p_{i,t}> \frac{1}{\sqrt{2\pi}}\frac{\mu_{max}/\tau_{max}}{1+(\mu_{max}/\tau_{max})^2} e^{-(\mu_{max}/\tau_{max})^2/2}=F(\frac{\mu_{max}}{\tau_{max}}) $. Note that $\tau_{max}>\frac{1}{12\sqrt{2}} $, then $N_t(\gamma,*)<\frac{1}{\tau_{max}^2}\leq A(\gamma)$.
	Therefore,
	\[
	\begin{aligned}
	\sum_{t \in \mathcal{T}(\gamma)} \mathbb{E}[\frac{1-p_{i,t}}{p_{i,t}}\mathbbm{1}\{ i_t=i_t^{*}, \theta_t(i)<y_t(i) \} ] &\leq \sum_{t \in \mathcal{T}(\gamma)}\frac{1}{F(\frac{\mu_{max}}{\tau_{max}})}  \mathbb{E}[\mathbbm{1}\{ i_t=i_t^{*}, N_t(\gamma,*)< A(\gamma) \} ] \\
	&\leq \frac{1}{F(\frac{\mu_{max}}{\tau_{max}})}  T(1-\gamma) A(\gamma) \gamma^{-1/(1-\gamma)}
	\end{aligned}
	\]
	In the subsequent analyses, we can assume that $N_t(\gamma,i)>\frac{1}{\tau_{max}^2}$, i.e. $ \theta_t(i) \sim \mathcal{N}(\hat{\mu}_t(i), \frac{1}{N_t(\gamma,i)}) $.
	The subsequent proof is in $3$ steps. 
	
	{\bfseries{Step 1}} We first prove that $\mathbb{E}[\frac{1}{p_{i,t}}] $ has an upper bound independent of $t$.
	
	Define a Bernoulli experiment as sampling from $\mathcal{N}(\hat{\mu}_t(i), \frac{1}{N_t(\gamma,i)})$, where success implies that $\theta_t(i)>y_t(i)$.  Let $G_t$ denote the number of experiments performed when the event $\{\theta_t(i)>y_t(i)\}$ first occurs. Then
	\[ \mathbb{E}[\frac{1}{p_{i,t}}]=\mathbb{E}[ \mathbb{E}[ G_t|\mathcal{F}_{t-1}] ]= \mathbb{E}[G_t]	 \]
	
	Let $z=\sqrt{\log r}+1$ ($r\geq 1$ is an integer ) and let $\text{MAX}_r$  denote the maximum of $r$ independent Bernoulli experiment. Then
	
	\begin{equation}
		\begin{aligned}
			\mathbb{P}(G_t \leq r) &\geq \mathbb{P}(\text{MAX}_r > \hat{\mu}_t(*)+ \frac{z}{\sqrt{N_t(\gamma,i)}} \geq y_t(i) )\\
			&= \mathbb{E}[ \mathbb{E}[  \mathbbm{1}\{ \text{MAX}_r > \hat{\mu}_t(*)+ \frac{z}{\sqrt{N_t(\gamma,i)}} \geq y_t(i)  \}| \mathcal{F}_{t-1} ] ]\\
			&= \mathbb{E}[  \mathbbm{1}\{ \hat{\mu}_t(*)+ \frac{z}{\sqrt{N_t(\gamma,i)}} \geq y_t(i) \} \mathbb{P}( \text{MAX}_r > \hat{\mu}_t(*)+ \frac{z}{\sqrt{N_t(\gamma,i)}} | \mathcal{F}_{t-1}) ]
		\end{aligned}
	\end{equation}
	Using Fact \ref{fact1}, 
	\begin{equation}
		\begin{aligned}
			\mathbb{P}( \text{MAX}_r > \hat{\mu}_t(*)+ \frac{z}{\sqrt{N_t(\gamma,i)}} | \mathcal{F}_{t-1})
			&\geq  1-(1-\frac{1}{\sqrt{2\pi}}\frac{z}{z^2+1}e^{-z^2/2})^r\\
			&= 1-(1-\frac{1}{\sqrt{2\pi}}\frac{\sqrt{\log r}+1}{(\sqrt{\log r}+1)^2+1}\frac{e^{-1/2-\sqrt{\log r}}}{\sqrt{r}}  )^r\\
			& \geq 1-e^{-\frac{\sqrt{r} e^{-\sqrt{\log r}}}{\sqrt{2\pi e} (\sqrt{\log r}+2)   }}
		\end{aligned}
	\end{equation}
	For any $r \geq e^{25}$, $ e^{-\frac{\sqrt{r} e^{-\sqrt{\log r}}}{\sqrt{2\pi e} (\sqrt{\log r}+2)   }} \leq \frac{1}{r^2} $.
	Hence, for any $r \geq e^{25}$, 
	\[ \mathbb{P}( \text{MAX}_r > \hat{\mu}_t(*)+ \frac{z}{\sqrt{N_t(\gamma,i)}} | \mathcal{F}_{t-1}) \geq 1- \frac{1}{r^2}. \]
	Therefore, for any $r \geq e^{25}$,
	\[ 
	\mathbb{P}(G_t \leq r) \geq (1-\frac{1}{r^2})\mathbb{P}( \hat{\mu}_t(*)+ \frac{z}{\sqrt{N_t(\gamma,i)}} \geq y_t(i)   ) 
	\]
	Next, we apply self-normalized  Hoeffding-type inequality\cite{garivier2011upper} to lower bound $\mathbb{P}( \hat{\mu}_t(*)+ \frac{z}{\sqrt{N_t(\gamma,i)}} \geq y_t(i)   )$.
	\[
	\begin{aligned}
		\mathbb{P}( \hat{\mu}_t(*)+\frac{z}{\sqrt{N_t(\gamma,i)}} \geq y_t(i) ) &\geq 1-\mathbb{P}( \hat{\mu}_t(*)+\frac{z}{\sqrt{N_t(\gamma,i)}} \leq \mu_t(*) )\\
		&\geq 1-\mathbb{P}( \hat{\mu}_t(*) - \ddot{\mu}_t(*)  \leq U_t(\gamma,i) -\frac{z}{\sqrt{N_t(\gamma,i)}})\\
		& \stackrel{(a)}{\geq} 1-\mathbb{P}( \hat{\mu}_t(*) - \ddot{\mu}_t(*)  <  -\frac{\sqrt{\log r}}{\sqrt{N_t(\gamma,i)}})\\
		&\geq 1-\frac{\log \frac{1}{1-\gamma}}{\log (1+\eta)}e^{- 2 \log r (1-\frac{\eta^2}{16})}
	\end{aligned}
	\]
	where (a) follows from the fact that $U_t(\gamma,i) -\frac{z}{\sqrt{N_t(\gamma,i)}} = \frac{\sqrt{(1-\gamma)\log \frac{1}{1-\gamma}}-1-\sqrt{\log r}}{\sqrt{N_t(\gamma,i)}} < - \frac{\sqrt{\log r}}{\sqrt{N_t(\gamma,i)}} $. Let $\eta= 2$, we get
	\[ \mathbb{P}( \hat{\mu}_t(*)+\frac{z}{\sqrt{N_t(\gamma,i)}} \geq y_t(i) ) \geq 1-\log \frac{1}{1-\gamma} \frac{1}{r^{1.5}} . \]		
	
	Substituting, for any $r > e^{25}$, 
	\begin{equation}
		\label{tmp1} \mathbb{P}(G_t\leq r) \geq 1-\log \frac{1}{1-\gamma} \frac{1}{r^{1.5}} -\frac{1}{r^2}
	\end{equation}
	Therefore,
	\[
	\begin{aligned}
		\mathbb{E}[G_t] &= \sum_{r=0}^{\infty}\mathbb{P}(G_t\leq r)\\
		&\leq 1+ e^{25} + \sum_{r>e^{25}} (\log \frac{1}{1-\gamma} \frac{1}{r^{1.5}} +\frac{1}{r^2})\\
		&\leq e^{25} + 3+ 3\log \frac{1}{1-\gamma}
	\end{aligned}
	\]
	This proves a bound of $\mathbb{E}[\frac{1}{p_{i,t}}] \leq e^{25} + 3+ 3\log \frac{1}{1-\gamma}$ independent of $t$.
	
	{\bfseries{Step 2}}. Define $L(\gamma)=\frac{144(1+\sqrt{2})^2\log(\frac{1}{1-\gamma}+e^{25})}{\Delta_T^2}$.
	We consider the upper bound of $\mathbb{E}[\frac{1}{p_{i,t}}]$ when $N_t(\gamma,i)>L(\gamma)$.
	
	\begin{equation}
		\begin{aligned}
			\mathbb{P}(G_t \leq r) &\geq \mathbb{P}(\text{MAX}_r > \hat{\mu}_t(*)+ \frac{z}{\sqrt{N_t(\gamma,i)}}-\frac{\Delta_t(i)}{6} \geq y_t(i) )\\
			&= \mathbb{E}[  \mathbbm{1}\{ \hat{\mu}_t(*)+ \frac{z}{\sqrt{N_t(\gamma,i)}} -\frac{\Delta_t(i)}{6} \geq y_t(i) \} \mathbb{P}( \text{MAX}_r > \hat{\mu}_t(*)+ \frac{z}{\sqrt{N_t(\gamma,i)}}-\frac{\Delta_t(i)}{6} | \mathcal{F}_{t-1}) ]
		\end{aligned}
	\end{equation}
	Now, since $N_t(\gamma,i)>L(\gamma)$,$\frac{1}{\sqrt{N_t(\gamma,i)}} < \frac{\Delta_t(i)}{12(1+\sqrt{2})\sqrt{\log( \frac{1}{1-\gamma}+e^{25})}}$.
	Therefore, for any $r \leq (\frac{1}{1-\gamma}+e^{25})^2$,
	\[ \frac{z}{\sqrt{N_t(\gamma,i)}}-\frac{\Delta_t(i)}{6} =\frac{\sqrt{\log r}+1}{\sqrt{N_t(\gamma,i)}}-\frac{\Delta_t(i)}{6} \leq -\frac{\Delta_t(i)}{12}. \]
	Using Fact \ref{fact1},
	\[ \mathbb{P}(\theta_t(i)> \hat{\mu}_t(i)-\frac{\Delta_t(i)}{12}|\mathcal{F}_{t-1} ) \leq 1-\frac{1}{2}e^{-N_t(\gamma,i)\frac{\Delta_t(i)^2}{288}}\geq 1-\frac{1}{2(1/(1-\gamma)+e^{25})^2}. \]
	This implies
	\[	
	\mathbb{P}( \text{MAX}_r > \hat{\mu}_t(*)+ \frac{z}{\sqrt{N_t(\gamma,i)}}-\frac{\Delta_t(i)}{6} | \mathcal{F}_{t-1}) \geq 1-\frac{1}{2^r(1/(1-\gamma)+e^{25})^{2r}}.
	\]
	Also, apply the fact that $ \frac{1}{\sqrt{N_t(\gamma,i)}} < \frac{\Delta_t(i)}{24}$ and the self-normalized  Hoeffding-type inequality,
	\[
	\begin{aligned}
	 \mathbb{P}(\hat{\mu}_t(*)+ \frac{z}{\sqrt{N_t(\gamma,i)}} -\frac{\Delta_t(i)}{6} \geq y_t(i) ) 
	 &\geq \mathbb{P}( \hat{\mu}_t(*)  \geq \mu_t(*) -\frac{\Delta_t(i)}{6}) )\\
	 &\geq 1-\log \frac{1}{1-\gamma}\frac{1}{(1/(1-\gamma)+e^{25})^6}.
	\end{aligned}
	\]
	Let $\gamma'=(\frac{1}{1-\gamma}+e^{25})^2$.
	Therefore,for any $1\leq r\leq \gamma'$,
	\[  \mathbb{P}(G_t \leq r) \geq 1-\frac{1}{2^r{\gamma'}^{2r}}-\log \frac{1}{1-\gamma}\frac{1}{{\gamma'}^6}. \]

	When $r \geq \gamma'>e^{25}$, we can use Equation \eqref {tmp1} to obtain,
	\[\mathbb{P}(G_t\leq r) \geq 1-\log \frac{1}{1-\gamma} \frac{1}{r^{1.5}} -\frac{1}{r^2}\]
	
	Combining these results,
	\[
	\begin{aligned}
		\mathbb{E}[G_t] &\leq \sum_{r=0}^{\infty}\mathbb{P}(G_t \geq r)\\
		&\leq 1+ \sum_{r=1}^{\gamma'}\mathbb{P}(G_t \geq r)+ \sum_{r=\gamma'}^{\infty}\mathbb{P}(G_t \geq r)\\
		&\leq 1+ \sum_{r=1}^{\gamma'} (\frac{1}{2^r{\gamma'}^{2r}}+\log \frac{1}{1-\gamma}\frac{1}{{\gamma'}^6} ) + \sum_{r=\gamma'}^{\infty} (\log \frac{1}{1-\gamma} \frac{1}{r^{1.5}} +\frac{1}{r^2})\\
		&\leq 1+ \frac{1}{\gamma'^2}+ \log \frac{1}{1-\gamma}\frac{1}{\gamma'^5}+ \frac{2}{\gamma'} +\log(\frac{1}{1-\gamma})\frac{3}{\sqrt{\gamma'}}\\
		&\leq 1+6(1-\gamma)\log \frac{1}{1-\gamma}.
	\end{aligned}
	\]
	
	Therefore, when $N_t(\gamma,i)>L(\gamma)$, it holds that 
	
	\[ \mathbb{E}[\frac{1}{p_{i,t}}]-1= \mathbb{E}[G_t]-1 \leq   6(1-\gamma)\log \frac{1}{1-\gamma}. \]
	
	{\bfseries{Step 3}} Let $ \mathcal{A}(\gamma,i)=\{t \in \{1,...,T\}:i_t=i_t^{*},N_t(\gamma,i) \leq L(\gamma) \}$ and $ C= e^{25}+\frac{1}{F(\frac{\mu_{max}}{\tau_{max}})} +12$. Combined with the case where $N_t(\gamma,i)<\frac{1}{\tau_{max}^2}$,
	\begin{equation}
		\label{step3}
		\begin{aligned}
			 &\sum_{t \in \mathcal{T}(\gamma)} \mathbb{E}[\frac{1-p_{i,t}}{p_{i,t}}\mathbbm{1}\{ i_t=i_t^{*}, \theta_t(i)<y_t(i) \} ] \\
			 &\leq  \sum_{t \in \mathcal{T}(\gamma) \cap \mathcal{A}(\gamma,i) } \mathbb{E}[\frac{1-p_{i,t}}{p_{i,t}}\mathbbm{1}\{ i_t=i_t^{*}, \theta_t(i)<y_t(i) \} ] +\sum_{t \in \mathcal{T}(\gamma) \setminus \mathcal{A}(\gamma,i) } \mathbb{E}[\frac{1-p_{i,t}}{p_{i,t}}\mathbbm{1}\{ i_t=i_t^{*}, \theta_t(i)<y_t(i) \} ]\\
			 &\leq |\mathcal{T}(\gamma) \cap \mathcal{A}(\gamma,i) | (e^{25}+3+3\log\frac{1}{1-\gamma}) +  \frac{1}{F(\frac{\mu_{max}}{\tau_{max}})}   T(1-\gamma) A(\gamma) \gamma^{-1/(1-\gamma)}
			  + \sum_{t \in \mathcal{T}(\gamma) \setminus \mathcal{A}(\gamma,i) } \mathbb{E}[\frac{1-p_{i,t}}{p_{i,t}} ] \\
			 &\leq T(1-\gamma)L(\gamma)\gamma^{-1/(1-\gamma)}(e^{25}+\frac{1}{F(\frac{\mu_{max}}{\tau_{max}})} +3+3\log\frac{1}{1-\gamma})+ 6T(1-\gamma)\log \frac{1}{1-\gamma} \\
			 &\leq CT(1-\gamma)L(\gamma)\gamma^{-1/(1-\gamma)}\log\frac{1}{1-\gamma}.
		\end{aligned}
	\end{equation}

	\subsection{Proofs of Lemma \ref{D1}}
	Let $M_t(\gamma,i)= \sum_{j=1}^{t} \gamma^{t-j} \mu_j(i) \mathbbm{1}\{ i_j=i \}$, 
	$\ddot{\mu}_t(\gamma,i)=\frac{M_t(\gamma,i)}{N_t(\gamma,i)} \in [0,1]$ is a convex combination of elements $\mu_j(i),j=1,...,t$.
		\[\begin{aligned}
		&|\mu_t(i)-\ddot{\mu}_t(\gamma,i)|\\
		&=\frac{1}{N_t(\gamma,i)} | \sum_{j=1}^{t} \gamma^{t-j}(\mu_j(i)-\mu_t(i))\mathbbm{1}\{ i_j=i \} | \\
		&= \frac{1}{N_t(\gamma,i)} | \sum_{j=1}^{t-D(\gamma)} \gamma^{t-j}(\mu_j(i)-\mu_t(i))\mathbbm{1}\{ i_j=i \} |
		+ \frac{1}{N_t(\gamma,i)} | \sum_{j=t-D(\gamma)}^{t} \gamma^{t-j}(\mu_j(i)-\mu_t(i))\mathbbm{1}\{ i_j=i \}
		\end{aligned}\]
	The first part can be bounded by $U_t(\gamma,i)$. Recall that the Assumption \ref{assumption_sigma}: 	There exits $\sigma>0$, for all $t,t^{'} \geq 1,1\leq i \leq K$, it holds that $| \mu_t(i)-\mu_{t^{'}}(i) | \leq \sigma |t-t^{'}|$. Therefore,
	\[  \frac{1}{N_t(\gamma,i)} | \sum_{j=t-D(\gamma)}^{t} \gamma^{t-j}(\mu_j(i)-\mu_t(i))\mathbbm{1}\{ i_j=i \}
	\leq \frac{\sigma D(\gamma)}{N_t(\gamma,i)} | \sum_{j=t-D(\gamma)}^{t} \gamma^{t-j}\mathbbm{1}\{ i_j=i \} \leq \sigma D(\gamma).
	 \]
	Hence, we get $|\mu_t(i)-\ddot{\mu}_t(\gamma,i)|\leq U_t(\gamma,i) +\sigma D(\gamma)$.

	\subsection{Proofs of Lemma \ref{new1}}
	The proof of Lemma \ref{new1} is almost the same as Lemma \ref{new}. Most of the results can be obtained directly from the proof of Lemma \ref{new}, and we will only present the different parts.
	In smoothly changing settings, $A(\gamma)=\frac{n^2\log(\frac{1}{1-\gamma})}{(\Delta/3-2\sigma D(\gamma))^2}$.
	We first bound $\mathbb{P}(G_t\leq r).$
	\begin{equation}
		\begin{aligned}
			\mathbb{P}(G_t \leq r) &\geq \mathbb{P}(\text{MAX}_r > \hat{\mu}_t(*)+ \frac{z}{\sqrt{N_t(\gamma,i)}} \geq y_t(i)-\sigma D(\gamma) )\\
			&= \mathbb{E}[  \mathbbm{1}\{ \hat{\mu}_t(*)+ \frac{z}{\sqrt{N_t(\gamma,i)}} \geq y_t(i)-\sigma D(\gamma)  \} \mathbb{P}( \text{MAX}_r > \hat{\mu}_t(*)+ \frac{z}{\sqrt{N_t(\gamma,i)}} | \mathcal{F}_{t-1}) ]
		\end{aligned}
	\end{equation}
	For any $r \geq e^{25}$,
	\[ 
	\mathbb{P}(G_t \leq r) \geq (1-\frac{1}{r^2})\mathbb{P}( \hat{\mu}_t(*)+ \frac{z}{\sqrt{N_t(\gamma,i)}} \geq y_t(i)-\sigma D(\gamma)   ) 
	\]
	We use self-normalized  Hoeffding-type inequality to lower bound $\mathbb{P}( \hat{\mu}_t(*)+ \frac{z}{\sqrt{N_t(\gamma,i)}} \geq y_t(i) -\sigma D(\gamma)   )$.
	\[
	\begin{aligned}
		\mathbb{P}( \hat{\mu}_t(*)+\frac{z}{\sqrt{N_t(\gamma,i)}} \geq y_t(i) -\sigma D(\gamma)) 
		&\geq 1-\mathbb{P}( \hat{\mu}_t(*)+\frac{z}{\sqrt{N_t(\gamma,i)}} \leq \mu_t(*) -\sigma D(\gamma))\\
		&\stackrel{(a)}{\geq} 1-\mathbb{P}( \hat{\mu}_t(*) - \ddot{\mu}_t(*)  \leq U_t(\gamma,i) -\frac{z}{\sqrt{N_t(\gamma,i)}})\\
		& \geq 1-\mathbb{P}( \hat{\mu}_t(*) - \ddot{\mu}_t(*)  <  -\frac{\sqrt{\log r}}{\sqrt{N_t(\gamma,i)}})\\
		&\geq 1-\frac{\log \frac{1}{1-\gamma}}{\log (1+\eta)}e^{- 2 \log r (1-\frac{\eta^2}{16})}
	\end{aligned}
	\]
	where (a) follows from the fact that $\mu_t(*)-\ddot{\mu}_t(*) \leq U_t(\gamma,i)+\sigma D(\gamma)$. Let $\eta= 2$, we get
	\[ \mathbb{P}( \hat{\mu}_t(*)+\frac{z}{\sqrt{N_t(\gamma,i)}} \geq y_t(i)-\sigma D(\gamma) ) \geq 1-\log \frac{1}{1-\gamma} \frac{1}{r^{1.5}} . \]	
	Therefore,
	\[
	\mathbb{E}[G_t] = \sum_{r=0}^{\infty}\mathbb{P}(G_t\leq r)
	\leq e^{25} + 3+ 3\log \frac{1}{1-\gamma}
	\]
	
	Next,let $L(\gamma)=\frac{144(1+\sqrt{2})^2\log(\frac{1}{1-\gamma}+e^{25})}{\Delta^2}$. We derive a tighter bound for $N_t(\gamma,i)> L(\gamma)$.
	
	\begin{equation}
		\begin{aligned}
			&\mathbb{P}(G_t \leq r)\\
			&\geq \mathbb{P}(\text{MAX}_r > \hat{\mu}_t(*)+ \frac{z}{\sqrt{N_t(\gamma,i)}}-\frac{\Delta_t(i)}{6} \geq y_t(i)-\sigma D(\gamma) )\\
			&= \mathbb{E}[  \mathbbm{1}\{ \hat{\mu}_t(*)+ \frac{z}{\sqrt{N_t(\gamma,i)}} -\frac{\Delta_t(i)}{6} \geq y_t(i)-\sigma D(\gamma) \} \mathbb{P}( \text{MAX}_r > \hat{\mu}_t(*)+ \frac{z}{\sqrt{N_t(\gamma,i)}}-\frac{\Delta_t(i)}{6} | \mathcal{F}_{t-1}) ]
		\end{aligned}
	\end{equation}
	
	Since $N_t(\gamma,i)>L(\gamma)$,$\frac{1}{\sqrt{N_t(\gamma,i)}} < \frac{\Delta}{12(1+\sqrt{2})\sqrt{\log( \frac{1}{1-\gamma}+e^{25})}}$.
	Therefore, for any $r \leq (\frac{1}{1-\gamma}+e^{25})^2$,
	\[ 
	\frac{z}{\sqrt{N_t(\gamma,i)}}-\frac{\Delta_t(i)}{6}  \leq -\frac{\Delta_t(i)}{12}. 
	\]
	Using Fact \ref{fact1},
	\[ 
	\mathbb{P}(\theta_t(i)> \hat{\mu}_t(i)-\frac{\Delta_t(i)}{12}|\mathcal{F}_{t-1} ) \leq 1-\frac{1}{2}e^{-N_t(\gamma,i)\frac{\Delta_t(i)^2}{288}}\geq 1-\frac{1}{2(1/(1-\gamma)+e^{25})^2}. 
	\]
	This implies
	\[	
	\mathbb{P}( \text{MAX}_r > \hat{\mu}_t(*)+ \frac{z}{\sqrt{N_t(\gamma,i)}}-\frac{\Delta_t(i)}{6} | \mathcal{F}_{t-1}) \geq 1-\frac{1}{2^r(1/(1-\gamma)+e^{25})^{2r}}.
	\]
	
	Also, apply the fact that $ \frac{1}{\sqrt{N_t(\gamma,i)}} < \frac{\Delta}{24}$ and the self-normalized  Hoeffding-type inequality,
	\[
	\begin{aligned}
		&\mathbb{P}(\hat{\mu}_t(*)+ \frac{z}{\sqrt{N_t(\gamma,i)}} -\frac{\Delta_t(i)}{6} \geq y_t(i)-\sigma D(\gamma) ) \\
		&\geq \mathbb{P}( \hat{\mu}_t(*)  \geq \mu_t(*)-\sigma D(\gamma)-\frac{\Delta_t(i)}{6}) \\
		&= 1-\mathbb{P}(\hat{\mu}_t(*)-\ddot{\mu}_t(*)  \leq \mu_t(*)-\ddot{\mu}_t(*)-\sigma D(\gamma)-\frac{\Delta_t(i)}{6})\\
		&\geq 1-\mathbb{P}(\hat{\mu}_t(*)-\ddot{\mu}_t(*)  \leq -\frac{\Delta_t(i)}{8})\\
		&\geq 1-\log \frac{1}{1-\gamma}\frac{1}{(1/(1-\gamma)+e^{25})^6}.
	\end{aligned}
	\]
	Let $\gamma'=(\frac{1}{1-\gamma}+e^{25})^2$.
	Combining these results,
	\[
	\begin{aligned}
		\mathbb{E}[G_t] &\leq \sum_{r=0}^{\infty}\mathbb{P}(G_t \geq r)\\
		&\leq 1+ \sum_{r=1}^{\gamma'}\mathbb{P}(G_t \geq r)+ \sum_{r=\gamma'}^{\infty}\mathbb{P}(G_t \geq r)\\
		&\leq 1+6(1-\gamma)\log \frac{1}{1-\gamma}.
	\end{aligned}
	\]
	
	Therefore, when $N_t(\gamma,i)>L(\gamma)$, it holds that 
	
	\[ \mathbb{E}[\frac{1}{p_{i,t}}]-1= \mathbb{E}[G_t]-1 \leq   6(1-\gamma)\log \frac{1}{1-\gamma}. \]
	
	Follows from Equation \eqref {step3}(step 3 in the proof of Lemma \ref{new}), let $ C= e^{25}+\frac{1}{F(\frac{\mu_{max}}{\tau_{max}})} +12$,we can get
	\[
	\sum_{t=1}^{T} \mathbb{E}[\frac{1-p_{i,t}}{p_{i,t}}\mathbbm{1}\{ i_t=i_t^{*}, \theta_t(i)<y_t(i) \} ] \leq CT(1-\gamma)L(\gamma)\gamma^{-1/(1-\gamma)}\log\frac{1}{1-\gamma}.
	\]

\end{document}